\documentclass[journal]{IEEEtran}
\usepackage{graphicx}
\usepackage{amssymb}
\usepackage{epstopdf}
\usepackage{rotating}
\usepackage{verbatim}
\usepackage{amsthm}
\usepackage{amsmath}
\usepackage{enumerate}
\usepackage{hyperref}
\usepackage{url}
\usepackage[]{mcode}
\usepackage{algorithm}
\usepackage{algorithmic}
\usepackage{cite}
\usepackage{array}

\newcommand{\D}{\mathcal{D}}
\newcommand{\Dnat}{\mathcal{D}_{nat}}
\newtheorem{definition}{Definition}
\newtheorem{lemma}{Lema}
\newtheorem{theorem}{Theorem}

\usepackage{bbm}

\begin{document}
\title{A metric for sets of trajectories that is\\practical and mathematically consistent}

\author{
\IEEEauthorblockN{Jos\'e Bento}
\IEEEauthorblockA{
jose.bento@bc.edu}\\
\IEEEauthorblockN{Jia Jie Zhu}
\IEEEauthorblockA{
zhuuv@bc.edu}\\
}

\maketitle
%
%
\begin{abstract}
Metrics on the space of sets of trajectories are important for scientists in the field of computer vision, machine learning, robotics, and general artificial intelligence.
However, existing notions of closeness between sets of trajectories are either
mathematically inconsistent or of limited practical use.
In this paper, we outline the limitations in the current 
mathematically-consistent metrics, which are based
on OSPA \cite{schuhmacher2008consistent}; and the inconsistencies
in the heuristic notions of closeness used in practice, whose
main ideas are common to the CLEAR MOT measures
\cite{keni2008evaluating} widely used in computer vision.
In two steps, we then propose a new intuitive metric
between sets of trajectories and address these limitations. 
First, we explain a solution that
leads to a metric that is hard to compute. Then we modify this
formulation to obtain a metric that is easy to compute while
keeping the useful properties of the previous metric.
Our notion of closeness
is the first demonstrating the following three features: the metric 1) can be quickly computed, 2) incorporates confusion of trajectories' identity in an
optimal way, and 3) is a metric in the mathematical sense.
\end{abstract}

\IEEEpeerreviewmaketitle

%
%

\section{Introduction} \label{sec:intro}

%
%
Similarity measures for sets of trajectories are very important.
In computer vision, they are used to
evaluate the performance of
multi-object tracking algorithms.
If $GT$ and $O$ are the ground-truth and output trajectories of a tracker, a similarity measure $\mathcal{D}$ can be used to distinguish a good tracker from a bad one, i.e. $\mathcal{D}(GT,O) = small$ implies $O$ is a good tracker.
In machine learning, algorithms such as \cite{ganti1999clustering,kleinberg2002approximation} and \cite{yianilos1993data}, can only cluster, classify and do a nearest neighbor
search on sets of trajectories, if we have a similarity measure.

Given their importance, one would expect that existing widely-used measures would be easy to compute and would not produce counter-intuitive results.
Surprisingly, this is not the case, leaving a critical
problem unsolved.

The main limitation of most similarity measures is that they are not a mathematical metric. For example,
the CLEAR MOT measures, widely used to evaluate the performance of trackers,
are based on a heuristic and are not a metric. 
If a measure does not satisfy, for example, the triangle inequality (cf. Section \ref{sec:setup_and_notation} ), then we cannot guarantee that two good trackers (good according to this measure) produce similar outputs, which is counter-intuitive.
In addition to this
inconsistency, the CLEAR MOT also produce other counter-intuitive results (cf. Section \ref{sec:limitations_MOT}). 
Furthermore, these results indirectly affect many
other similarity measures that internally use the CLEAR MOT, or similar heuristics,
to precompute an association between two sets of trajectories (e.g. \cite{blackman1986multiple,rothrock2000performance,bar2004estimation,gorji2011performance}).

However, even existing measures that are a metric, and that, we believe, are all variants of OSPA, can produce unreasonable results in simple scenarios (cf. Section \ref{sec:diff_scenarios}). We use OSPA-ST to refer to OSPA applied to sets of trajectories since it was not originally defined for this setting.

On the positive side, both the CLEAR MOT and the OSPA-based metrics are easy to define and have a computation time that scales well with the number and length of trajectories. These two attributes are part of the reason why the CLEAR MOT are popular.

In this paper we introduce the first measures that are simultaneously a mathematical metric, do not replicate the counter-intuitive results of OSPA-ST or CLEAR MOT and can also be quickly computed (cf. Section \ref{sec:main_results}).

In the next section, we use an example to introduce the primary technical
challenge in defining a measure for sets of trajectories and how this is addressed by OSPA-ST, the
CLEAR MOT, and our metrics.
We focus on OSPA-ST and the CLEAR MOT
because they give a simple but powerful overview of the main ideas leading up to our work.  
We include a detailed list of related work in Section \ref{sec:relwork}.

%

%
%

\vspace{-0.3cm}
\section{Main challenge: the association problem} \label{sec:diff_scenarios}
\vspace{-0.5cm}
\begin{figure}[h]
\begin{center}
\includegraphics[height=3.8cm]{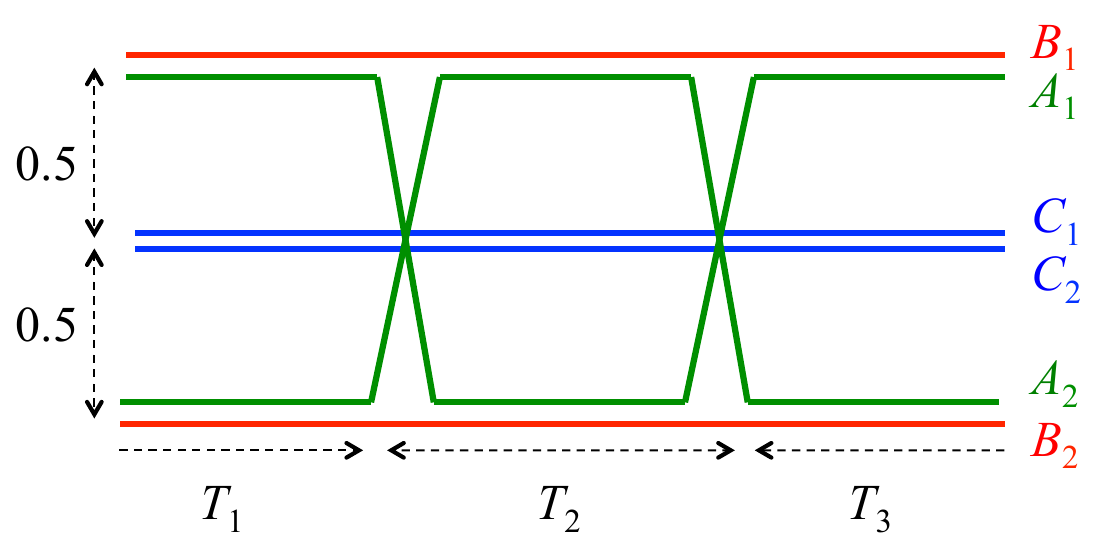}
\vspace{-0.3cm}
\caption{Two people, $A_1$ and $A_2$, move along a line and are followed by two trackers that generate outputs $B_1$, $B_2$, $C_1$ and $C_2$. For visualization purposes, trajectories that are close
		are actually on top of each other.}
\label{fig:toy_example}
\end{center}
\end{figure}
\vspace{-0.3cm}

Figure \ref{fig:toy_example} shows three sets of trajectories $A$, $B$ and $C$. Set $A$ has the ground truth trajectories of two people that start at a distance $1.0$ from each other and, twice, quickly exchange positions; $B$ and $C$ are the trajectories that two different tracking
algorithms output. We represent time on the $x$-axis (left to right) in frames and space on the $y$-axis (top to bottom). For convenience, $T_1$, $T_2$ and $T_3$ are normalized to sum to $1.0$.
We seek a distance measure $\mathcal{D}$ for sets of trajectories to determine if $B$ (and/or $C$) is a good tracker.

If $A$ and $B$ only contained one trajectory, say $A_1$ and $B_1$, we could define $\mathcal{D}(A,B)$ 
by computing the distance between $A_1$ and $B_1$ at each frame and returning the average distance over frames.
Since this is not the
case, the next idea is to define $\mathcal{D}$ in two steps. \emph{Step 1}: determine, for each frame, if we should compute the distance between $A_1$ and $B_1$ and between $A_2$ and $B_2$ {\it or} if we should compute the distance between 
$A_1$ and $B_2$ and between $A_2$ and $B_1$.
We refer to step 1 as \emph{finding an association} between $A$ and $B$. If, at a certain frame, we decide to compute the distance between $A_1$ and $B_2$ and between $A_2$ and $B_1$, then we say that, at that frame, \emph{we associate $A_1$ with $B_2$ and $A_2$ with $B_1$}. Later, we formally define associations using permutations (See Section \ref{sec:setup_and_notation}).
\emph{Step 2}: given this association,
compute the average distance between associated points.

The main problem in defining $\mathcal{D}$ is to establish an association, i.e. step 1 above.
Loosely speaking, typical approaches involve choosing an association that makes
the average distance between associated points small.
We can divide different approaches in three cases. The association between $A$ and $B$ {(i)} can change freely from frame to frame; {(ii)} cannot change; and {(iii)} can change from frame to frame
but we pay a (smaller/higher) cost for (smaller/bigger) changes. We call this last cost
a \emph{switching cost}. We name the total (or average) distance between associated points the \emph{distance cost}.

In (i), a ``good tracker'' tracks position accurately even if it does not track
peoples' identities (ID) correctly.
At each frame, we associate $A$ to $B$ such that the sum of the distance
between associated points is minimal. For Fig \ref{fig:toy_example}, we obtain an average distance between $A$ and $B$ (over frames and objects) that is $\mathcal{D}(A,B)=0$ and an average distance between $A$ and $C$ that is $\mathcal{D}(A,C)=0.5$. Tracker $B$ is better than $C$. Notice that from $T_1$ to $T_2$
track $B_1$ changes the person it is tracking.
We call this as an \emph{identity switch}.

In (ii), the tracker's output $B_1$ must either be associated with $A_1$ for all frames
or associated with $A_2$ for all frames. $B_2$ is associated
with the other trajectory in $A$.
Between these two possibilities, we again choose the one
that minimizes the average distance (over frame and objects).
Now $\mathcal{D}(A,B)=\min\{T_2, 1-T_2\}$ and
$\mathcal{D}(A,C)=0.5$. Tracker $C$ can be as good as tracker $B$ (for $T_2 = 1/2$), which is counter-intuitive, because $C$'s output is never close to the ground-truth. 

In (iii), a ``good tracker'' must trade off some position accuracy (distance cost) for some ID accuracy (switching cost).
The switching cost penalizes identity switches.
It is also evident that category (iii) includes (i) and (iii) as particular cases of extreme tradeoffs.

The most widely used measures
fall in category (iii).
In computer vision, the prototypical example
is MOTP, one of the CLEAR MOT measures. 
At each frame, MOTP heuristically tries to maintain the association between $A$ and $B$ as close as possible to
the association made in the previous frame. It makes corrections to the
association only if the association created in the previous frame applied to the current frame produces distances between matched points that are larger than a certain threshold $thr_{MOT}$.
In the first frame, MOTP uses the association that minimizes the total distance.
(cf. Definition \ref{def:MOT_association} in Appendix \ref{app:proof_of_th_mota_bad_association} for a formal definition).
By changing this threshold one can have different results. We always have $\mathcal{D}(A,C) = 0.5$. If $thr_{MOT}$ is small,
then $\mathcal{D}(A,B) = 0$, because we switch association twice. However, if 
$thr_{MOT}$ is large, $\mathcal{D}(A,B) = T_2$, because we never change from the association made in the first frame. 
Unfortunately, this heuristic leads to MOTP not being a metric
and to other counter-intuitive results (cf. Section \ref{sec:limitations_MOT}). 
In addition, all metrics (based on OSPA) fall in category (ii)
because their associations are fix in time.
Thus, the applications where they produce intuitive
results are limited. We give another example of how OSPA-ST
can produce counter-intuitive results in Appendix \ref{app:extra}.

Like the CLEAR MOT, our measures fall in category (iii). However, ours do not compute the association between $A$ and $B$ heuristically or sequentially. Rather, we solve a \emph{global optimization problem} such that the sum of the distance cost and the switching cost over all time frames is minimized. Thus we avoid counter-intuitive results and obtain a metric.

To appreciate the difference between optimizing associations globally or sequentially, notice that, for Figure \ref{fig:toy_example}, MOTP either does not change the association across frames or changes it twice, from $T_1$ to $T_2$ and from $T_2$ to $T_3$. Regardless of $T_1$, $T_2$ and $T_3$, MOTP never considers changing association just once.
This happens because MOTP uses a fixed threshold to decide when to change associations. Hence, it fails to explore possibly better ways in which to compare $A$ and $B$.

%
%

\vspace{-0.3cm}
\section{Setup and notation} \label{sec:setup_and_notation}

We denote the collection of all finite sets of finite trajectories by $S$. We reserve the letters $A$, $B$ and $C$ to represent finite sets of finite trajectories. 
$A_i$ is the $i^{th}$ trajectory in $A$. Each trajectory $A_i$ is a finite set of time-state pairs
$(t,x)$ with time $t\in \mathbb{N}$ and state $x \in \mathbb{R}^p$. We focus on $t\in \mathbb{N}$ but it is possible to generalize our results to continuous time.
We use $A_i(t)$ to represent the state of the $i^{th}$ trajectory in $A$ at time $t$.

When $A$ and $B$
are defined for all time instants and have the same number of trajectories, we can express $\D(A, B)$ 
using very simple notation.
However, $A_i(t)$ might
not be defined for all values of $t$. In addition, $A$ and $B$ might have a different number of trajectories.
Thus, for mathematical convenience, we define a symbol $*$ and the following extension procedure.

Given $A = \{A_i\}^{m_1}_{i=1}$ and
$B = \{B_i\}^{m_2}_{i=1}$,
we let $m = m_1 + m_2$ and define
$A^+ = \{A^+_i\}^{m}_{i=1}$ and
$B^+ = \{B^+_i\}^{m}_{i=1}$ as follows.
Let $T$ be the largest time
index for which either $A$ or $B$ have
trajectories with defined states.
If $i \leq m_1$,
then, for all $t$ such that $A_i(t)$ is defined, we set $A^+_i(t) = A_i(t)$. Otherwise, $A^+_i(t) = *$.
If $m_1< i \leq m$ then $A^+_i(t) = *$ for all $t \in \{1,...,T\}$. We call these $A^+_i$, $*$-only trajectories. We define $B^+$
in the same way. Now $A^+$ and $B^+$ both have $m$ trajectories and their states are defined in the extended set $\mathbb{R}^p \cup \{*\}$ for all $t\in \{1,...,T\}$. We call $A^+$ and $B^+$ extended sets of trajectories. Figure \ref{fig:notation}-(b) exemplifies this procedure. For example, $A^+_1$ agrees with $A_1$ for all $t$
except for $t=3$ for which $A_1$ is not defined and $A^+_1 = *$. 

The meaning of an instant $t$ for which $A_i(t)$ is not defined, i.e. $A^+_i(t) = *$, depends on the application, e.g. it might mean an occlusion. Other
interpretations are possible, e.g., an object has yet to come into existence. We refrain for adhering to a particular interpretation of $*$. Our use of $*$
resembles \cite{ristic2011metric}, where a null symbol $\emptyset$ is used.

We use $\mathcal{D}:S\times S\mapsto \mathbb{R}^+_0$ to represent distance measures on $S$. If $A,B \in S$, then
$\mathcal{D}(A,B)$ measures the distance between $A$ and $B$.
The main goal of this work is to introduce mathematical metrics that are also easy to compute and produce intuitive results.
Recall that, for any $A,B,C \in S$, a metric $\D$ must satisfy the properties (i) {\bf (non-negativity)} $\mathcal{D}(A,B) \geq 0$, (ii) {\bf (coincidence)} $\mathcal{D}(A,B) = 0$ iff $A = B$, (iii) {\bf (symmetry)} $\mathcal{D}(A,B) = \mathcal{D}(B,A)$ and (iv) {\bf (sub-additivity)} $\mathcal{D}(A,C) \leq \mathcal{D}(A,B) + \mathcal{D}(B,C)$. 

Our definition of $\mathcal{D}$ in the following sections needs two ingredients: an association between extended sets of trajectories and a distance between extended trajectories' states.

\emph{We formally define an association between two sets of $m$ elements using permutations}.
A permutation $\sigma: i \mapsto \sigma_i$ is a bijective map from $\{1,...,m\}$ to itself. $\Pi$ denotes the set of all permutations. 
We define the composition of $\sigma, \sigma' \in \Pi$ as $\sigma' \circ \sigma: i \mapsto \sigma'_{\sigma_i}$. $\sigma^{-1}$ is the inverse map of the bijection $\sigma$.

An association $\sigma \in \Pi$ between $A$ and $B$ for which $\sigma_i = j$ tells us that, to compute $\D(A,B)$, we will use distances between the states of $A^+_i$ and the states of $B^+_j$ (See e.g. Def. \ref{eq:DefOSPA}). In this case, we say that $\sigma$ associates trajectory $A^+_i$ to trajectory $B^+_j$.
Because computing $\D(A,B)$ might involve computing 
distances between different pairs of trajectories $A_i$, $B_j$ at different points in time, we extend the terminology ``association'' to also mean a sequence of permutations. We define $\Pi^T = \{\Sigma : \Sigma = (\Sigma(1),\Sigma(2),...,\Sigma(T)), \Sigma(t) \in \Pi \; \forall t\}$ as the set of all length-$T$ sequences of associations. $\Sigma_i(t)$ is the image of $i$ by the map $\Sigma(t) \in \Pi$. We define $\Sigma^{-1} \triangleq (\Sigma(1)^{-1},...,\Sigma(T)^{-1}) \in \Pi^T$ and
$\Sigma' \circ \Sigma \triangleq (\Sigma'(1)\circ \Sigma(1),...,\Sigma'(T)\circ \Sigma(T)) \in \Pi^T$.
An association sequence $\Sigma \in \Pi^T$ between $A$ and $B$ for which $\Sigma_i(t) = j$ tells us that, to compute $\D(A,B)$, we will use the distance between $A^+_i(t)$ and $B^+_j(t)$ (See e.g. Def. \ref{def:MOT_metric}). In this case, we say that $\Sigma$ associates the state $A^+_i(t)$ to the state $B^+_j(t)$ at time $t$.

We use $d:\mathbb{R}^p\times\mathbb{R}^p \mapsto \mathbb{R}^+_0$ to indicate a distance between state elements. 
We define the extended distance $d^+:\mathbb{R}^p\cup\{*\} \times \mathbb{R}^p\cup\{*\} \mapsto \mathbb{R}^+_0$ such that for every
$x,y \in \mathbb{R}^p$ we have (i) $d^+(x,y) = \min\{2M,d(x,y)\}$, 
(ii) $d^+(x,*)= d^+(*,x)= M > 0$ and (iii) $d^+(*,*) = 0$. 
Property (iii) guarantees that placeholders by themselves
do not change the distance between $A$ and $B$.
Property (ii) defines a cost $M$ for comparing two states, one of which is not defined. For example, a larger $M$ might indicate a higher penalty for a missed or false track. Property (i) of $d^+$ is a technicality that makes $d^+$ be a metric.
We define $D^{AB}(t) \in \mathbb{R}^{m \times m}$ with element $(ij)$ equal to $d^+(A^+_i(t),B^+_j(t))$.

In the context of computer vision tracking, having 
$m_2$ $*$-only trajectories in $A^+$, i.e. trajectories will all states equal to $*$, and $m_1$ *-only trajectories in $B^+$, allows two important types of associations. 
\vspace{-0.1cm}
\begin{table}[h!]
\begin{center}
 \begin{tabular}{||c |c|| } 
 \hline
 Symbols &  Meaning  \\
\hline
 \hline
$A$, $A^+$ & Set of trajectories and set of extended trajectories\\ 
 \hline
$A_i(t)$, $A^+_i(t)$ & State of $i^{th}$ trajectory in $A$, and $A^+$, at time $t$\\ 
 \hline
  $*$ & Symbol to mark states outside Euclidean space\\ 
 \hline
  $S$ &  Set of all possible trajectories \\
 \hline
$ \mathcal{D}$ & Distance between trajectory sets\\
 \hline
 $d,d^+$ & Distance between states and between extended states\\
 \hline
  $M$ & Cost between state * and state in  Euclidean space\\  
 \hline
  $D^{AB}_{ij}(t)$ & Distance $d^+$ between extended states $A^+_i(t)$ and $B^+_j(t)$\\  
 \hline
  $\Pi$ & Set of all possible permutations (perms.)\\  
 \hline
  $\sigma$, $\sigma^{-1}$, $\sigma_i$ & Permutation, its inverse, the image of $i$ by the map $\sigma$\\  
 \hline
   $\Sigma$, $\Sigma^{-1}$ & Sequence of permutations, sequence of their inverses\\  
    \hline
   $\Sigma_i(t)$ & The image of $i$ by the map $\Sigma(t)$ in the sequence $\Sigma$\\  
 \hline
  $\Pi^T$ & Set of all possible sequences of perms. of length $T$  \\  
 \hline
  $\mathcal{K}(\Sigma)$ & Switching cost of permutation sequence  $\Sigma$ \\  
 \hline
  $w, W$ & Doubly stochastic matrix (d.s.m.), sequence of d.s.m.\\  
 \hline
   $W_{ij}(t)$ & Entry $(ij)$ of the $t^{th}$ d.s.m. of the sequence $W$\\  
 \hline
  $\mathcal{P}$ & Set of all possible d.s.m. \\
 \hline
$\mathcal{P}^T$ & Set of all possible sequences of d.s.m. of length $T$\\  
 \hline
    $\|.\|$, $\dagger$, \bf{tr} & Operators for matrix norm, transpose, trace \\  
 \hline
\end{tabular}
\end{center}
\vspace{-0.6cm}
\end{table}
In computing $\D(A,B)$, we might not
want to use any distance between a ground-truth trajectory in $A$, say $A_1$, and any reconstructed trajectories in $B$. We get this by associating $A^+_1$ to a $*$-only trajectory in $B^+$. The points in $A_1$ are \emph{miss detections}. A reconstructed trajectory, say $B_1$, might not be related to any ground truth trajectories in $A$. We represent this by associating $B^+_1$ to a $*$-only trajectory in $A^+$. $B_1$ is a \emph{spurious trajectory}. To create the possibility that all ground truth trajectories might be missed and all tracker trajectories might be spurious, we need at least $m_2$ *-only trajectories in $A^+$ and $m_1$ *-only trajectories in $B^+$.

With a particular application in mind, we might want to distinguish occlusions and no-target or penalize missed and false tracks differently. In this case, we might introduce multiple different placeholder symbols during the extension procedure, say $*$ and $\#$, 
and extend the definition of $d^+$ to penalize
different symbol comparisons differently, e.g. $d^+(*,\#) = M_1$, $d^+(*,x) = M_2$, $d^+(\#,x) = M_3$, $d^+(*,*) = d^+(\#,\#) = 0$. 

Not all interpretations and uses of these symbols make sense. Let $B$ be the output of a tracker and $A$ the ground truth. If $B^+$ has one $*$-only trajectory, a miss detection, and $A$ has one $*$-only trajectory, an always occluded object, then $\mathcal{D}(A,B) = 0$, even tough there is a cardinality error. This occurs because the meaning of $*$ in $A$ and $B$ is inconsistent. We could have used a  different symbol for miss detections, $*$, and occluded ground truth, $\#$. Furthermore, it is unreasonable to let always-hidden  objects be part of the ground truth, i.e. let $A$ have one $*$-only trajectory. Otherwise, a tracker is only good if it can find the number of objects behind a curtain. 
\vspace{-0.1cm}
\begin{figure}[h]
\begin{center}
\includegraphics[height=4.8cm]{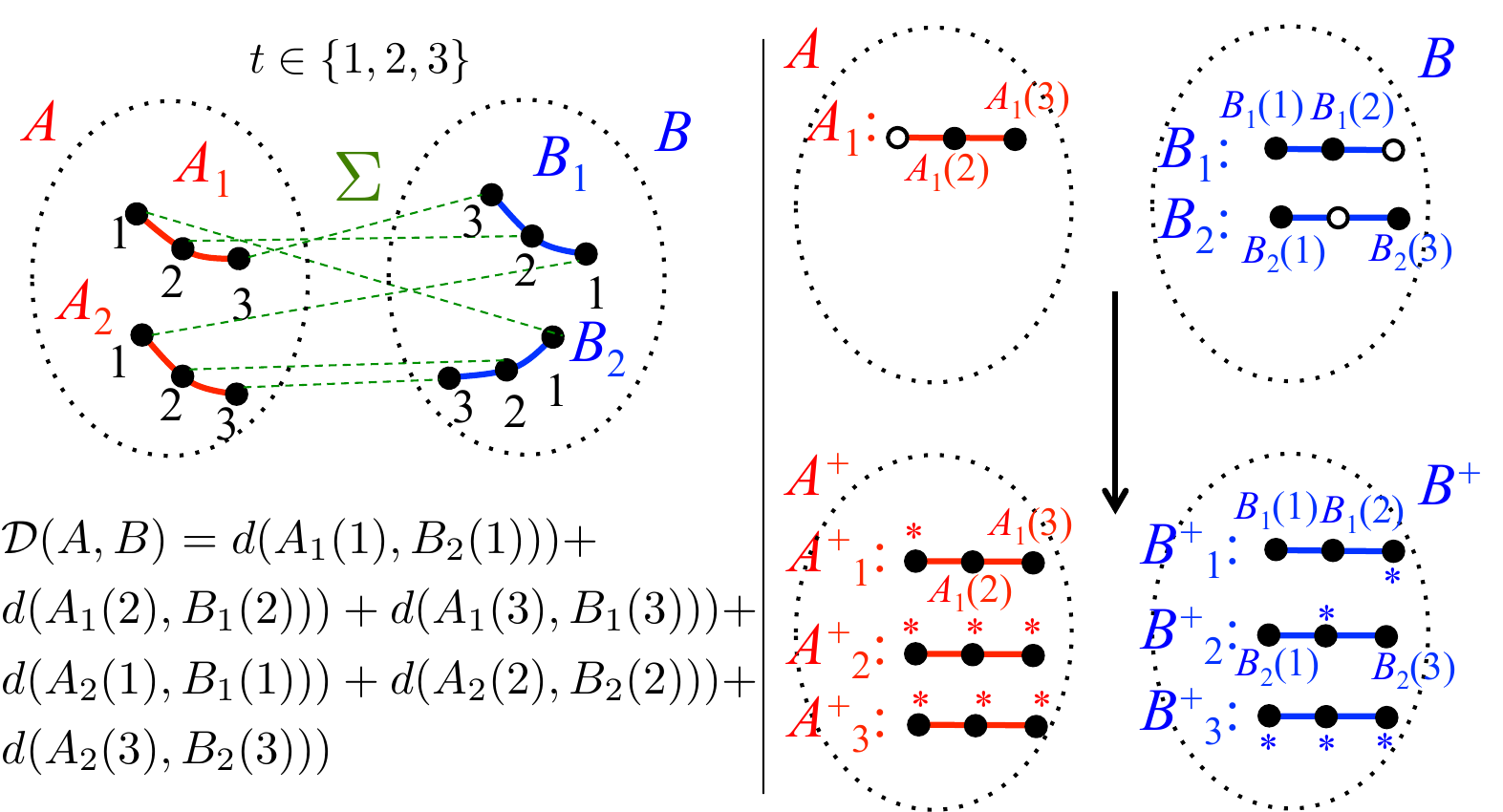}
\put(-250,130){\rotatebox{0}{  (a)}}
\put(-70,130){\rotatebox{0}{\small  (b)}}
\vspace{-0.5cm}
\caption{ (a) Role of associations in computing $\mathcal{D}$; (b) Extension procedure.}
\vspace{-0.3cm}
\label{fig:notation}
\end{center}
\end{figure}

Figure \ref{fig:notation}-(a) illustrates how we
can use an association $\Sigma$ between $A$ and $B$ and distance $d^+$ between
state elements to define a distance 
$\mathcal{D}(A,B)$. In Fig. \ref{fig:notation}-(a) we have dropped $^+$ for
clarity. 
In this figure we have $\Sigma = (\Sigma(1), \Sigma(2),\Sigma(3))$
where $\Sigma(1) = (2,1)$, $\Sigma(2) = (1,2)$, $\Sigma(3) = (1,2)$.

In this paper we denote the set of doubly stochastic matrices as
$\mathcal{P} = \{w \in \mathbb{R}^{m \times m}: 
w^{\dagger} {\bf 1} = {\bf 1},  w {\bf 1} = {\bf 1}, w \geq 0\}$.
We also define
$\mathcal{P}^T = \{W: W = (W(1),...,W(T)), W(t) \in \mathcal{P} \; \forall t\}$ as the set of all length-$T$ sequences of doubly stochastic matrices.

Without specification, $\|\cdot\|$ denotes the Euclidean norm, $\|\cdot\|_1$ denotes the 1-norm, $\dagger$ means transpose and \text{\bf tr} matrix trace.

The above table gathers the most important notation used.

%
%
%
%

\section{Proposed metrics} \label{sec:new_metrics}

Here we introduce two novel families of metrics. We study their properties in Section \ref{sec:main_results}.

Consider a map $\mathcal{K}\hspace{-0.1cm}: \hspace{-0.1cm}\Pi^T \mapsto \mathbb{R}^+_0$ that
gives a score, a switching cost, to sequences of associations and a map $d: \mathbb{R}^p \times \mathbb{R}^p  \mapsto \mathbb{R}^+_0$
that computes distances between trajectories' states. Recall that $D^{AB}_{ij}(t)=d^+(A^+_i(t),B^+_j(t))$ depends on $d$'s definition.
\begin{definition}\label{def:dnat}
The {\bf natural distance} induced by $\mathcal{K}$ and $d$ between two sets of trajectories
is a map $\mathcal{D}_{nat}: \mathcal{S}\times\mathcal{S}\mapsto\mathbb{R}^+_0$ such that for any $A,B \in \mathcal{S}$
{
\begin{align}\label{eq:defDnat}
\mathcal{D}_{nat}(A,B) \hspace{-0.1cm}=\hspace{-0.1cm} \min_{\Sigma \in \Pi^T} \Big\{ \mathcal{K}(\Sigma)
+\sum^T_{t=1} \sum^{m}_{i=1} D^{AB}_{i\, \Sigma_i(t)}(t) \Big\}. 
\end{align}
}
\end{definition}

A natural choice
for $d$ is the Euclidian metric, that is,
$d(x,x') = \|x-x'\|_2$.
One intuitive choice for $\mathcal{K}$ is
$\mathcal{K}_{count}(\Sigma) = \alpha \sum^{T-1}_{t=1} \mathbb{I}(\Sigma(t+1) \neq \Sigma(t) )$, $\alpha > 0$, that basically counts the number of times we change association. 
We give other definitions for $\mathcal{K}$ in Section \ref{sec:main_results}.

Generally, computing $\mathcal{D}_{nat}$ is a combinatorially hard problem
so we introduce a new metric that uses doubly stochastic matrices instead of permutations as associations between $A$ and $B$ \footnote{By the Birkhoff--von Neumann theorem we know that $\mathcal{P}$ is smallest convex set that contains all permutations (represented as permutation matrices).}.
Consider any norm $\|.\|$ on the space of matrices.
\begin{definition}\label{def:dcomp}
The {\bf natural computable distance} induced by $\|.\|$ and $d$ between two sets of trajectories
is a map $\mathcal{D}_{comp}: \mathcal{S}\times\mathcal{S}\mapsto\mathbb{R}^+_0$ such that for any $A,B \in \mathcal{S}$
{
\begin{align}\label{eq:defDcomp}
&\mathcal{D}_{comp}(A,B) =  \min_{W\in \mathcal{P}^T} \Big\{ \sum^{T-1}_{t = 1} \|W(t+1) - W(t)\|\nonumber\\
&+ \sum^T_{t=1} \sum^m_{i,j=1} W_{ij}(t)\, D^{AB}_{ij}(t) \Big\}. 
\end{align}
}
\end{definition}

A few important observations follow. First, notice that both measures
minimize the sum of a distance cost plus a switching cost.
Second, notice that we are defining a family of
measures, not just two measures. Since $d$, $\mathcal{K}$ and $\|.\|$ are generic, \emph{our definitions can easily
incorporate a scaling factor in front of each term of equations
\eqref{eq:defDnat} and \eqref{eq:defDcomp}}. Therefore,
we can control the relative importance of the switching and distance costs.

Third, $\mathcal{P}$ is a convex set and any norm $\|.\|$ is convex
in $\mathcal{P}$, so computing $\mathcal{D}_{comp}$ amounts to solving
a convex optimization problem, for which there are easy-to-use
and efficient packages, e.g. CVX \cite{cvx}. Later we show one choice of $\|.\|$, among several, that reduces \eqref{eq:defDcomp} to a linear program (LP).
From this LP, we can approximate $\mathcal{D}_{nat}$
by forcing the variables to be in $\{0,1\}$ and then use techniques
such as branch-and-bound, e.g. available in MATLAB and CPLEX.

Both $\mathcal{D}_{comp}$ and $\mathcal{D}_{nat}$ solve a global optimization problem to find the best
tradeoff between distance cost and switching cost. As explained in Section \ref{sec:diff_scenarios}, this is in contrast to most
measures used in practice such as MOTP, as explained below.
\begin{definition}\label{def:MOT_metric}
The CLEAR MOT's distance measure for evaluating tracking precision is called MOTP and is defined as 
\begin{equation}\small
\mathcal{D}_{MOTP}(A,B) = \sum^T_{t=1}\sum^m_{i=1} D^{AB}_{i \, \Sigma_{\text{MOT}_i}(t)}(t).
\end{equation}
\end{definition}
In the definition above, $\Sigma_{\text{MOT}} \in \Pi^T$ is the sequence of  associations
that CLEAR MOT builds heuristically (explained in Section \ref{sec:diff_scenarios}
and formally in Appendix \ref{app:proof_of_th_mota_bad_association}).

We also know that $\mathcal{D}_{nat}$ and $\mathcal{D}_{comp}$ are more flexible than OSPA-ST, where associations cannot change with time. To see this, compare \eqref{eq:defDnat} and \eqref{eq:defDcomp} with the definition of OSPA-ST below.
\begin{definition}
The OSPA-ST metric is defined as
\begin{equation}\label{eq:DefOSPA}\small
\mathcal{D}_{OSPA-ST}(A,B) = \min_{\sigma \in \Pi} \sum^T_{t=1}\sum^m_{i=1}  D^{AB}_{i \, \sigma_{i}}(t).
\end{equation}
\end{definition}

Our metrics include OSPA-ST as a particular case.
Indeed, we recover OSPA-ST metric if we choose $\mathcal{K}$ as: $\mathcal{K}_{OSPA-ST}(\Sigma) = 0$ if $\Sigma = (\sigma,\sigma,...,\sigma)$, for some $\sigma \in \Pi$, and
$\mathcal{K}_{OSPA-ST}(\Sigma) = \infty$ otherwise.
In this case, we can compute $\mathcal{D}_{comp}$ in polynomial time, just like for OSPA-ST, using e.g. the Hungarian algorithm \cite{kuhn1955hungarian} (cf. \cite{schuhmacher2008consistent}).

The doubly stochastic matrices $W(t)$ of  $\mathcal{D}_{comp}$, can be transformed into permutations using a projection procedure that amounts to solving a LP. However, we cannot obtain doubly stochastic matrices from $\mathcal{D}_{nat}$.
We can use the associations from $\mathcal{D}_{nat}$ and $\mathcal{D}_{comp}$ to improve the computation of other heuristic measures of multi-target tracking performance, e.g. the number of ID switches or track purity.

Since $d^+(*,*)=0$, the $*$-only trajectories add zero distance cost when associated with each other and so our metrics do not produce irrelevant switches between
$*$-only trajectories.
In addition, although $\mathcal{K}$ is a function of $\Sigma$ but not of $A$ and $B$,
when a particular application is in mind, we can have $\mathcal{K}$ treat switches between different indices differently. We might, for instance,
impose that confusing the ID of tracks $1$ and $20$ (e.g. different-team players) should be heavily penalized but confusing the ID of
tracks $1$ and $4$ (e.g. same-team players) should not. Also, we can
impose that a switch between a detected person and a $*$, a non-detected person, should have a different cost than a switch between two detected persons.

Finally, using duality, and \emph{from now on using more compact matrix notation}, we can equivalently define $\mathcal{D}_{comp}$ (cf. \eqref{eq:defDcomp})
as
{
\begin{align*}
\mathcal{D}_{comp}(A,B) =&  \hspace{-0.0cm}\min_{W \in \mathcal{P}^T}  \sum^T_{t=1} \text{\bf tr}\left(W(t)^{\dagger} D^{AB}(t)\right) \\
&\text{subject to } \sum^{T-1}_{t = 1} \|W(t+1) - W(t)\| \leq \alpha
\end{align*}
}
for some $\alpha$. Acting by analogy, $\mathcal{D}_{nat}$ can also be re-defined in a similar way.
In other words, we do not need to worry about the sum of distance cost and switching cost.
%
%
%
%
\vspace{-0.2cm}
\section{Related work} \label{sec:relwork}

Except for our work, we found no similarity measure
for sets of trajectories that is both mathematically consistent (a metric) and, at the same time, is useful and can deal with identity switches, i.e. it allows time-dependent associations.
We focus our discussion around these characteristics.
However, several of the ideas we review can
be incorporated into our metrics to define new variants
that compete with past work in settings not discussed, given the
scope of this paper.

Our work is related to the general problem of defining a distance
between two sets, however, this is a topic too vast to review in this paper alone.
In \cite{deza2009encyclopedia}, the reader can find many of those
distances. In this article, the two sets $A$ and $B$ are sets of trajectories, and this limits the scope of our discussion.
In the simplest case, when trajectories have only one vector,
typical definitions compute an average or sum distance between
all pairs of elements from $A$ and $B$ or just from a few pairs, e.g. \cite{fujita2013metrics}. Knowing which pairs to use when computing distances requires a
procedure that matches elements of $A$ with elements of $B$, e.g.
\cite{schuhmacher2008consistent,gardner2014measuring}.
In general, however, each trajectory is composed of a set of vectors indexed
by time, which limits our discussion even more.

To the best of our knowledge, the most rigorous works on metrics
for sets of trajectories are based on \cite{schuhmacher2008consistent}.
In \cite{schuhmacher2008consistent} the authors propose a distance for sets of vectors, the optimal sub-pattern
assignment metric (OSPA), and explain its advantages over other distances in the context of multi-object filters.
The OSPA metric has better sensitivity than the Hausdorff metric to cardinality
differences between sets and does not lead to complicated interpretations as does the optimal mass transfer metric (OMAT) \cite{hoffman2004multitarget}, proposed to address
the limitations of the Hausdorff metric. OSPA, a metric for sets
of vectors, can be defined from any distance between two vectors.

All the spin-offs of \cite{schuhmacher2008consistent} focus on designing a metric for two sets of trajectories for the purpose of evaluating the performance of tracking algorithms. We recall however that, apart from tracking, many applications in machine learning and AI benefit if we work with a metric rather than a similarity measure that is not a metric. For example, the algorithms mentioned
in the introduction \cite{ganti1999clustering,kleinberg2002approximation, yianilos1993data} can only cluster, classify and 
find nearest neighbors if $\mathcal{D}$ is a metric.
None of the
spin-offs of \cite{schuhmacher2008consistent} allow associations
that change with time and hence suffer from the
same restricted applicability as OSPA (cf. Section \ref{sec:diff_scenarios}-(ii)).

The OSPA-T metric \cite{ristic2011metric} for sets of trajectories is defined in two steps. First, we optimally match full tracks in $A$ to full tracks in $B$ while considering that
tracks have different lengths and can be incomplete. Second,
we assign labels to each track based on this match and compute
the OSPA metric for each time using a new metric $d$ for pairs of
vectors. $d$ uses both the vectors' components and their labels.
Finally, we sum the OSPA values across all time instants.
Although the first step alone defines a
metric, the authors in \cite{vu2014new} point out that the full two-step procedure that describes OSPA-T is not a metric.

\cite{vu2014new} defines the OSPAMT
to be a metric and be more reliable than OSPA-T.
The OSPAMT metric also optimally matches
complete trajectories in $A$ to complete trajectories in $B$ but
unlike OSPA-T allows us to match one trajectory in $A$ to multiple trajectories in $B$ (and vice-versa).
The authors make this design choice to not penalize a tracker
when it outputs only one track for two objects that
move closely together.

Some extensions to OSPA incorporate measurement uncertainty.
The Q-OSPA metric \cite{QOSPA2013} weighs the distance between pairs of points by the product of their certainty and by adding a new term that is proportional to the product of the uncertainties. The H-OSPA
metric \cite{nagappa2011incorporating} uses
OSAP with distributions as states
instead of vectors and uses the Hellinger distance between
distributions instead of the Euclidean distance between vectors.
Both works only allow sets that contain points, not trajectories.
However, it is easy to combine them with \cite{ristic2011metric}
or \cite{vu2014new} and get a metric for sets of trajectories.

The papers above are relatively recent, as is
the search for mathematical metrics for sets of trajectories.
However, researchers studying
multi-object tracking have been using similarity measures
for sets of trajectories to rank tracking algorithms
for a while. Reviewing all the
work done in this area is impossible. Especially because evaluating tracking performance has many challenges other than
the problem of defining a similarity measure. See 
\cite{ellis2002performance,milan2013challenges} for some
examples. Nonetheless,
we mention various works that relate to our problem.
We emphasize that none of the following works defines
a metric mathematically.

Most of our paper is about finding an association
between the elements in $A$ and $B$ (cf. Section \ref{sec:diff_scenarios}).
In \cite{fridling1991performance}, which is expanded in \cite{drummond1992ambiguities}, the authors are one of the first to identify this as the central problem in defining $\mathcal{D}$, although some of their ideas draw from the much earlier Ph.D. thesis of one of the authors \cite{Drummond75}. They propose
an one-to-one
association between $A$ and $B$ but
this association is optimally computed independently at every time instant. Also, there is no discussion about the number of association changes that might occur.

The CLEAR MOT are used widely in part
because they allow control over the number of times that 
the association between
$A$ and $B$ changes.
It appears that \cite{colegrove1996performance} is one of the first that allow this. Like the CLEAR MOT,
\cite{colegrove1996performance} uses a sequential matching procedure that tries to keep the association from the previous time instant if possible. 
The association that \cite{colegrove1996performance} uses at each time $t$
is not one-to-one optimal like in \cite{fridling1991performance} or in the CLEAR MOT. Rather, the authors use a simple thresholding rule to decide which elements to associate. The idea of using a simple
threshold rule to associate $A$ and $B$ has survived until relatively recently. For example,  
\cite{yin2007performance} match a full trajectory in $A$ to
a whole trajectory in $B$ if they are close in space for a sufficiently long
time interval. The authors in \cite{bashir2006performance} use a similar
thresholding rule.

Shortly after \cite{fridling1991performance} and \cite{drummond1992ambiguities},
some of the same authors discuss again the problem of associating $A$ to $B$.
\cite{drummond1999methodologies} proposes four different
methods for this problem.
First: $A$ and $B$ are optimally
matched independently at each time instant,
possibly generating different associations at every time instant.
Second: the association between $A$ and $B$ cannot change with
time. Like in the OSPA-based metrics, this restricts the measure's
applicability (cf. Section \ref{sec:diff_scenarios}).
Third: allow the association between $A$ and $B$ to change in time
but only in special circumstances. Unfortunately, as they point out, this leads
to an NP-hard problem. Fourth: use heuristics to minimize the number of changes in association. Some of these ideas resemble those used in the CLEAR MOT.

Measuring trackers' performance is not just about 
defining similarity measures for sets of trajectories,
or only concerned with establishing associations between $A$ and $B$.
If we have already matched the elements of $A$ and $B$, which quantity should we compute from this match?
A typical quantity that researchers calculate is
the average or total distance between matched points. Our metrics are a combination of this
quantity with the switching cost in the $A$-$B$ association.
Distance-related quantities apart,
researchers also compute quality measures
for the match itself, e.g. the number of trajectories in $A$ that
do not match to any trajectory in $B$ (spurious tracks and missed tracks) or the number of times the match between $A$ and $B$ changes. Quantities like 
track purity and coalescence are often also computed once a match is
established. \cite{blackman1986multiple,rothrock2000performance,bar2004estimation} and \cite{gorji2011performance} list many other measures of similarity between $A$ and $B$ and measures of quality for individual tracks.

It is worth mentioning a few non-mainstream works.
\cite{pingali1996performance} defines a distance based on comparing the occurrence of
special discrete events in $A$ and $B$. 
\cite{edward2009information} proposes an information theoretic similarity measure. Finally, \cite{porikli2004trajectory} proposes a similarity measure based on hidden Markov models to allow
unequal temporal sampling rates in the trajectories.
%
%
\vspace{-0.0cm}
\section{Limitations of the CLEAR MOT} \label{sec:limitations_MOT}

In the context of visual tracking, the next lemma shows that there are situations in which the average tracking distance cost over time gets arbitrarily close to zero with time while the CLEAR MOT produces an association $\Sigma_{\text{MOT}}$ that fails to see that. In other words, the CLEAR MOT can be counter-intuitive. 
Since none of the definitions of Section
\ref{sec:new_metrics} deal with time averages, in this section we use the following two definitions.

For any $\Sigma \in \Pi^T$, $A,B \in \mathcal{S}$ we let
\begin{equation}\label{eq:swi_for_MOTP}
swi(\Sigma) = \frac{1}{T-1}\sum^{T-1}_{t=1} \mathbbm{I}\{\Sigma(t) \neq \Sigma(t+1)\}
\end{equation}
denote the empirical frequency of identity switches, i.e. association changes, where $\mathbbm{I}\{\cdot\}$ is the indicator function. Let
\begin{equation}\label{eq:dist_for_MOTP}
dist(\Sigma,A,B) = \frac{1}{T}\sum^T_{t=1} \sum^m_{i=1} 
D^{AB}_{i \, \Sigma_i(t)}(t)
\end{equation}
denote the average distance between $A$ and $B$ under $\Sigma = (\Sigma(1),...,\Sigma(T))$.
Note that, for example, $\mathcal{D}_{MOTP} = T \times dist(\Sigma_{MOTP},A,B)$.

\begin{lemma}\label{th:MOT_inconsistent}
For any threshold $thr_{MOT} \hspace{-0.1cm}>\hspace{-0.1cm} 0$, there exists trajectory sets $A, B \in \mathcal{S}$ and an association $\Sigma \in \Pi^T$ such that 
\begin{align*}
& dist(\Sigma,A,B) < \mathcal{O}(1/T),\\
& swi(\Sigma) =0 \text{ and},\\ 
& dist(\Sigma_{\text{MOTP}},A,B) > m\,thr - \mathcal{O}(1/T).
\end{align*}

\end{lemma}

The following lemma shows that, for any $thr_{MOT} > 0$,
the MOTP measure is mathematically inconsistent.
\begin{lemma}\label{th:MOPT_no_metric}
MOTP is not a metric for any threshold $thr_{MOT} \hspace{-0.1cm}>\hspace{-0.1cm} 0$.
\end{lemma}
The proofs of these two lemmas are in Appendix
\ref{app:proof_of_th_mota_bad_association}.

It is also possible to prove that MOTA, the other widely used measure in
the CLEAR MOT, does not define a metric. See 
\url{www.jbento.info/papers/metriccompanion.pdf}.

%
%
%
%

\section{Properties of our metrics} \label{sec:main_results}

The metrics we introduce do not produce the counter-intuitive results of
OSPA-ST or MOTP because they control the tradeoff between switches
and distances in a globally optimal way.
In addition, we also show that they are a metric
under mild conditions on $\mathcal{K}$, $\|.\|$ and $d$.

\begin{definition}\label{th:property_of_Sigma}
A map $\mathcal{K}: \Pi^T \mapsto \mathbb{R}^+_0$ is a {\bf permutation measure} if it satisfies the following three properties
(i) $\mathcal{K}(\Sigma) = 0$ if and only if $\Sigma$ is
constant $\Sigma = (\sigma,\sigma,...,\sigma)$ for some $\sigma \in \Pi$, (ii) $\mathcal{K}(\Sigma^{-1}) = \mathcal{K}(\Sigma)$ and (iii) $\mathcal{K}(\Sigma \circ \Sigma') \leq \mathcal{K}(\Sigma) + \mathcal{K}(\Sigma')$. 
\end{definition}

A few observations are in order. Clearly, if $\mathcal{K}$ is a permutation measure so is $\alpha \mathcal{K}, \alpha > 0$. Also, property (iii) requires
that all permutation matrices have the same dimensions, otherwise we cannot
talk about their composition.  In practice, this can be naturally achieved if e.g. all trajectories have a common maximum observation time $T$, sets of trajectories have a common maximum number of trajectories $M$, and permutation matrices have dimensions $2M\times2M$. Finally, Def. \ref{th:property_of_Sigma} implies that $\mathcal{K}$ is invariant under reindexing. Indeed,  by property (iii) followed by property (i) $\mathcal{K}( \Sigma \circ (\sigma, \sigma, \dots, \sigma)) \leq \mathcal{K}( \Sigma ) + 	\mathcal{K}( (\sigma, \sigma, \dots, \sigma)) = \mathcal{K}( \Sigma )$. On the other hand, by property (iii) followed by property (i) followed by property (ii) $\mathcal{K}( \Sigma ) = \mathcal{K}( \Sigma \circ (\sigma, \sigma, \dots, \sigma) \circ 	(\sigma, \sigma, \dots, \sigma)^{-1}) \leq \mathcal{K}( \Sigma \circ 		(\sigma, 	\sigma, \dots, \sigma) ) + \mathcal{K}( (\sigma, \sigma, \dots, \sigma)^{-1}) = 		\mathcal{K}( \Sigma \circ (\sigma, \sigma, \dots, \sigma))$. Therefore, $\mathcal{K}( \Sigma \circ (\sigma, \sigma, \dots, \sigma)) = \mathcal{K}( \Sigma )$. A similar argument shows that $\mathcal{K}(  (\sigma, \sigma, \dots, \sigma)\circ\Sigma ) = \mathcal{K}( \Sigma )$. This invariance is part of the reason why we can later show that, under certain conditions, $\mathcal{D}_{{nat}}$ and $\mathcal{D}_{{comp}}$ are metrics.

One straightforward choice for $\mathcal{K}$ is to count the number of times that the association between $A$ and $B$ changes.
\begin{theorem}\label{th:K_count}
Let $
\mathcal{K}_{count}(\Sigma) = \sum^{T-1}_{t=1} \mathbb{I}(\Sigma(t+1) \neq \Sigma(t) )$.
$\mathcal{K}_{count}$ is a permutation measure.
\end{theorem}

Another two desirable choices for $\mathcal{K}$ are (a) the function that sums the minimum number of transpositions
to go from one permutation to the next and (b) the function
that sums the number of adjacent transpositions to go
from one permutation to the next.
In what follows $k_{Cayley}(\sigma)$ gives the minimum number of transpositions
to obtain the identity permutation from $\sigma \in \Pi$ and $k_{Kendall}(\sigma)$ gives the number of adjacent transposition that the
Bubble-Sort algorithm performs when sorting $\sigma$ to obtain the identity permutation. The
Cayley distance goes back to \cite{cayley1849lxxvii} and the Kendall distance to\cite{kendall1938new}.

\begin{theorem}\label{th:K_trans}
Let $\mathcal{K}_{trans}(\Sigma) = \sum^{T-1}_{t=1} k_{Cayley}(\Sigma(t+1)\circ\Sigma(t)^{-1})$.  $\mathcal{K}_{trans}$ is a permutation measure.
\end{theorem}

The proof of Theorems \ref{th:K_count} and \ref{th:K_trans} is in Appendix \ref{app:proof_for_K_propreties}.

Although many intuitive choices for $\mathcal{K}$ satisfy
properties (i), (ii) and (iii) of Definition \ref{th:property_of_Sigma}, some natural ones do not. For example, given a $\beta \geq 1$ we might want to define
$\mathcal{K}_{maxcount}$ as
$\mathcal{K}_{maxcount}(\Sigma) = \mathcal{K}_{count}(\Sigma)$ if $\mathcal{K}_{count}(\Sigma) \leq \beta$ and 
$\mathcal{K}_{maxcount}(\Sigma) = \infty$ if $\mathcal{K}_{count}(\Sigma) > \beta$ (we can replace $\infty$ by some very large number if we want to be technical about the range of $\mathcal{K}$ being $\mathbb{R}^+_0$).
This $\mathcal{K}$ forces us not to create an association between $A$ and $B$ more intricate than a certain amount $\beta$, something natural to desire.
The following is proved in Appendix \ref{app:necessary_conditions_for_D_nat_metric}.
\begin{theorem} \label{th:K_maxcount_not_proper}
$\mathcal{K}_{maxcount}$ does not satisfy (iii) in Definition \ref{th:property_of_Sigma}. Thus, it is not a permutation measure.
\end{theorem}

The following theorem is similar to Theorem \ref{th:K_maxcount_not_proper} and its proof is omitted.
\begin{theorem}\label{th:K_adjtrans}
Let $\mathcal{K}_{adjtrans}(\Sigma) =  \sum^{T-1}_{t=1} k_{Kendall}(\Sigma(t+1)\circ\Sigma(t)^{-1})$.  $\mathcal{K}_{adjtrans}$ does not satisfy (iii) in Definition \ref{th:property_of_Sigma}. Thus, it is not a permutation measure.\end{theorem}

We now state our first main result.
\begin{theorem} \label{th:Dnaturalismetric}
If $\mathcal{K}$ is a permutation measure and $d$ is a metric, then
the map $\mathcal{D}_{nat}$ induced by them is a metric on $\mathcal{S}$.
\end{theorem}

Even if a function $\mathcal{K}$ violates some of the properties in Definition \ref{th:property_of_Sigma}, it could still induce a $\mathcal{D}_{nat}$ that is a metric. However,
we often find that, from a set of associations
$\Sigma$ and $\Sigma'$ that violate Definition \ref{th:property_of_Sigma},
we can build three sets of trajectories $A$, $B$ and $C$
that violate some of the properties required of a metric. This is the case, for example, for $\mathcal{D}_{nat}$ induced by $\mathcal{K}_{maxcount}$
and $d$ equal to the Euclidean metric.
\begin{theorem} \label{th:K_maxcount_leads_to_not_metric}
The map $\mathcal{D}_{nat}$, induced by $\mathcal{K} = \mathcal{K}_{maxcount}$ and the Euclidean distance $d$, is not a metric.
\end{theorem}
The proof of this theorem is in Appendix \ref{app:necessary_conditions_for_D_nat_metric}.
In this sense, Definition \ref{th:property_of_Sigma} is required for $\mathcal{D}_{nat}$ to be a family of metrics.

We now focus on $\mathcal{D}_{comp}$.
\begin{definition}\label{th:property_of_matrix_norm}
A matrix norm $\|.\|$ is a {\bf switching norm} if 
for any four matrices $w_1,w_2,w'_1,w'_2 \in \mathcal{P}$
\begin{equation}\label{eq:propofnormforDcomp}
\|w'_2w_2 - w'_1w_1\| \leq \|w'_2 - w'_1\| + \|w_2 - w_1\|.
\end{equation}
\end{definition}

The following Lemma gives sufficient conditions for a metric $\|.\|$
to satisfy property \eqref{eq:propofnormforDcomp}. See Appendix \ref{app:proof_that_norms_for_Dcomp_are_many} for the proof.
\begin{lemma} \label{th:many_norms_satisfy_D_comp}
If $\|.\|$ is a sub-multiplicative norm and $\|W\| \leq 1$ for all
$W \in \mathcal{P}$ then $\|.\|$ is a switching norm.
\end{lemma}
This lemma implies, for example, that the 1-norm, $\infty$-norm and 
spectral norm for matrices are valid choices for $\|.\|$.

We now state our second most important result.
\begin{theorem} \label{th:D_comp_is_metric}
If $\|.\|$ is a switching norm and $d$ is a metric then the
 map $\mathcal{D}_{comp}$ induced by $\|.\|$ and $d$ is
 a metric on $\mathcal{S}$.
\end{theorem}
The proof of this theorem is in Appendix \ref{app:proof_that_D_comp_is_metric}.

The use of $\|.\|_1$ (matrix norm) in $\mathcal{D}_{comp}$ is extremely useful
because it induces the changes of association to be sparse
and, as the next theorem shows, reduces $\mathcal{D}_{comp}$ to solving a linear program. Recall that all LPs can be solved in
polynomial time \cite{khachiian1979polynomial}. The theorem's proof is in Appendix \ref{app:proof_that_D_comp_equals_LP}.
\begin{theorem} \label{th:dcompisLP}
For any $A, B \in \mathcal{S}$, the metric $\mathcal{D}_{comp}(A,B)$
induced by the matrix $1$-norm and any $d$ can be computed (in polynomial time) by solving a linear program.
\end{theorem}
This LP can be a made sparse if we impose that for every $(i,j,t)$ such that $A_i(t)$ and $B_j(t)$ are distant we must have $W_{ij}(t) = 0$, i.e.
we cannot associate distant points.
Sparsity allows us to reduce the effective number of
optimization variables in \eqref{eq:defDcomp} and speedup computations further.

To end this section, let us explicitly include a scaling factor $\alpha > 0$
in the definition of $\mathcal{D}_{comp}$. To be specific,
\begin{equation}\label{eq:Dcompwithscaling}
\mathcal{D}_{comp}(A,B) = \min_{W \in \mathcal{P}^T}\;\;\alpha\times swi(W) + dist(W,A,B),
\end{equation}
where, changing the definition of \eqref{eq:swi_for_MOTP} and \eqref{eq:dist_for_MOTP} in Section \ref{sec:limitations_MOT},
\begin{align}
swi(W) &= \sum^{T-1}_{t=1}\|W(t+1)-W(t)\| \text{ and }\label{eq:swi_for_Dcomp}\\
dist(W,A,B) &= \sum^T_{t=1} \text{\bf tr}(W(t)^{\dagger} D^{AB}(t))\label{eq:dist_for_Dcomp}.
\end{align}

If we compute 
$\mathcal{D}_{comp}$ for different $\alpha$'s we obtain different pairs of
values $dist$ and $swi$. We can obtain more pairs of values if
we compute $dist$ and $swi$ using as
 $W$ the permutation matrices representing
 the association produced by the CLEAR MOT or OSPA-ST. Even
using MOTP alone, we can compute different values
for $swi$ and $dist$ by changing $thr_{MOT}$.

If we fix $A$ and $B$, we can display all these different pairs on a 2D scatter plot where each point is a pair $(dist,swi)$ evaluated on a different $W \in \mathcal{P}^T$. Using such a scatter plot is a great way to assess the performance of a tracker $B$
on ground-truth $A$ under different similarity measures or when
we do not know which $\alpha$ or $thr_{MOT}$ to use
to compute $\mathcal{D}_{comp}$ and $\mathcal{D}_{MOTP}$.
For a fixed $A$ and $B$, a good measure $\mathcal{D}$
generates pairs closer to the origin.
A very good measure $\mathcal{D}$ generates pairs only on the Pareto frontier of this scatter plot. 
Not surprisingly, the pairs $(dist,swi)$ that $\mathcal{D}_{comp}$
produces for different values of $\alpha$ generate this Pareto frontier. 
This follows from a well-known result in convex optimization theory that we state in Theorem \ref{th:optimal_tradeoff} (see \cite{boyd2004convex} for a proof).
In this context, $\mathcal{D}_{comp}$ is the best metric we can hope.
In particular, none of the examples where MOTP or OSPA-ST
produce counter-intuitive results, like the counter example behind the proof of
Lemma \ref{th:MOT_inconsistent}, lead to $\mathcal{D}_{comp}$ giving counter-intuitive results.
The reader unfamiliar with convex optimization can jump to Section \ref{sec:num_res_run_time}.

Let $\Omega$ be a convex set, e.g. $\mathcal{P}^T$, and let
$f$ and $g$ be two convex functions in $\Omega$, e.g. $swi(.)$ and $dist(.,A,B)$. Let $\mathcal{R} = \{(s,t) \in \mathbb{R}^2: s \geq f(x) \text{ and } t \geq g(x), x \in \Omega \}$. This set could be, for example, the points in our scatter plot plus points with worse costs. Let $\partial \mathcal{R}$ be the boundary of $\mathcal{R}$. Since $\mathcal{R}$ is convex, $\partial \mathcal{R}$ is its Pareto frontier. For instance, $\partial \mathcal{R}$ could be the Pareto frontier we discussed above.
Let $p(\alpha) \in \mathbb{R}^2$ be the curve of $(f,g)$ pairs generated
by solving $\min_{x\in \Omega } f(x) + \alpha g(x) $ for different values of $\alpha > 0$. This could be a tradeoff curve of $(swi,dist)$ pairs generated by $\mathcal{D}_{comp}$ for different $\alpha$'s.
\begin{theorem} \label{th:optimal_tradeoff}
If $(s, t) \in \partial \mathcal{R}$ then either $(s,t) = p(\alpha_0)$
for some $\alpha_0 > 0$ or $(s,t)$ is a convex combination of
$p(\alpha_0)$ and $p(\alpha_1)$ for some $\alpha_0,\alpha_1 > 0$.
\end{theorem}
%

%
%
%
%
%

\section{Numerical results: $\mathcal{D}_{comp}$ computation time}
\label{sec:num_res_run_time}

In practice, there are many ways to compute
$\mathcal{D}_{comp}$, like when solving a convex optimization problem. 
To illustrate how easy it is to get code working, we have a simple Matlab code for
$\mathcal{D}_{comp}$ in \url{www.jbento.info/papers/metriccompanion.pdf}. In the same document,
we also have a simple Matlab code to estimate $\mathcal{D}_{nat}$.
However, to explore the limits of practical performance,
we coded a non-trivial solver in C using $parADMM$, an
implementation of the Alternating Direction Method of Multipliers (ADMM) 
that is available at \url{github.com/parADMM/engine}. $parADMM$ has
good performance in practice \cite{hao2016testing}.
The ADMM is known to scale well, and its modular nature makes it easy to research future variants of $\mathcal{D}_{comp}$ without having to re-write much code. Furthermore, for strongly convex problems, its optimally-tuned convergence rate is as fast as that of the fastest first-order method \cite{francca2016explicit}.

\cite{boyd2011distributed} is a good introduction to the ADMM and its applications.

In Figure \ref{fig:profilingcodeforDcomp} we plot run-time in computing $\D_{comp}$ against the total duration $T$ of the input data for a different number of free association variables per time instant $t$. We use the term ``free association variables''
because, as explained after Thm. \ref{th:dcompisLP}, a few variables $W_{ij}(t)$ can be set to zero to sparsify the problem and save computation time.
We choose the euclidean distance for $d$ and the one-norm for matrix norm $\|.\|$ in the switching cost. Similar run-times hold for other metrics. The run-time is computed for randomly generated $A$ and $B$ on a single core of a $1.4$GHz Intel Core i5 MacBook Air. ADMM always converged to $1$\% accuracy in less than $150$ iterations.
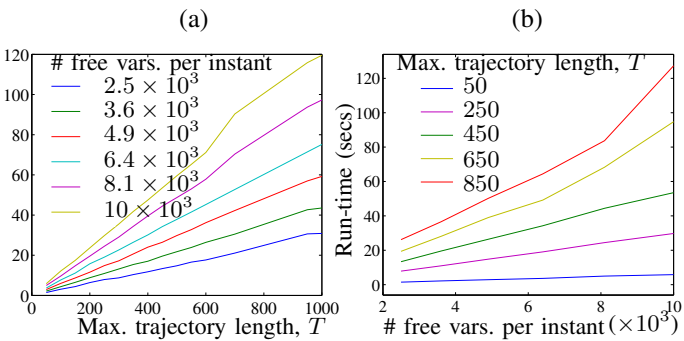
\begin{figure}[b]
\vspace{-0.5cm}
\begin{center}
\includegraphics[trim=0.8cm 0.4cm 0.6cm 0cm, clip=true,height=3.8cm]{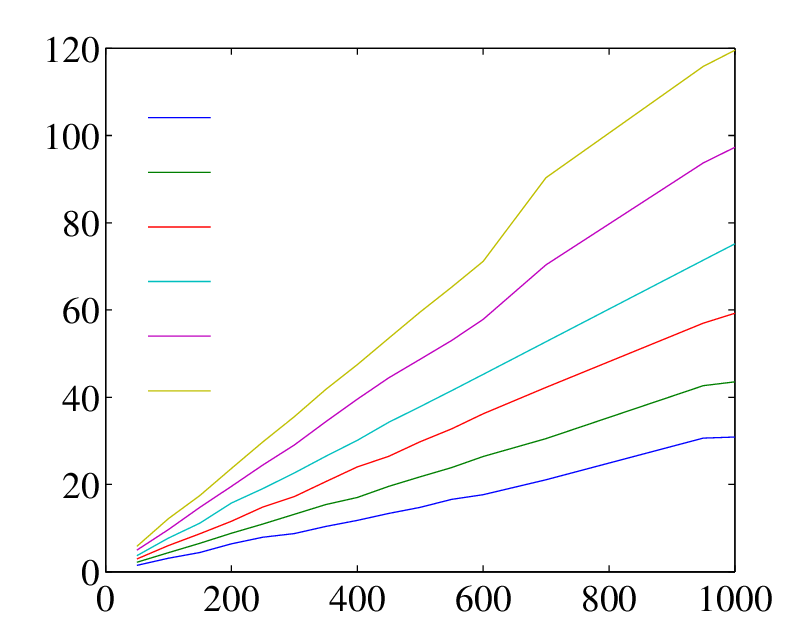}\;\;
\put(-70,110){\rotatebox{0}{ (a)}}
\put(-07.5,23){\rotatebox{90}{\small Run-time (secs)}}
\put(-105,-06){\rotatebox{0}{\small Max. trajectory length, $T$}}
\put(-115,94){\rotatebox{0}{\small \# free vars. per instant}}
\put(-95,85){\rotatebox{0}{\small $2.5 \times 10^3$}}
\put(-95,76){\rotatebox{0}{\small $3.6 \times 10^3$}}
\put(-95,67){\rotatebox{0}{\small   $4.9 \times 10^3$}}
\put(-95,57){\rotatebox{0}{\small   $6.4 \times 10^3$}}
\put(-95,48){\rotatebox{0}{\small   $8.1 \times 10^3$}}
\put(-95,38){\rotatebox{0}{\small   $10 \times 10^3$}}
\includegraphics[trim=0.8cm 0.4cm 0.8cm 0cm, clip=true,height=3.8cm]{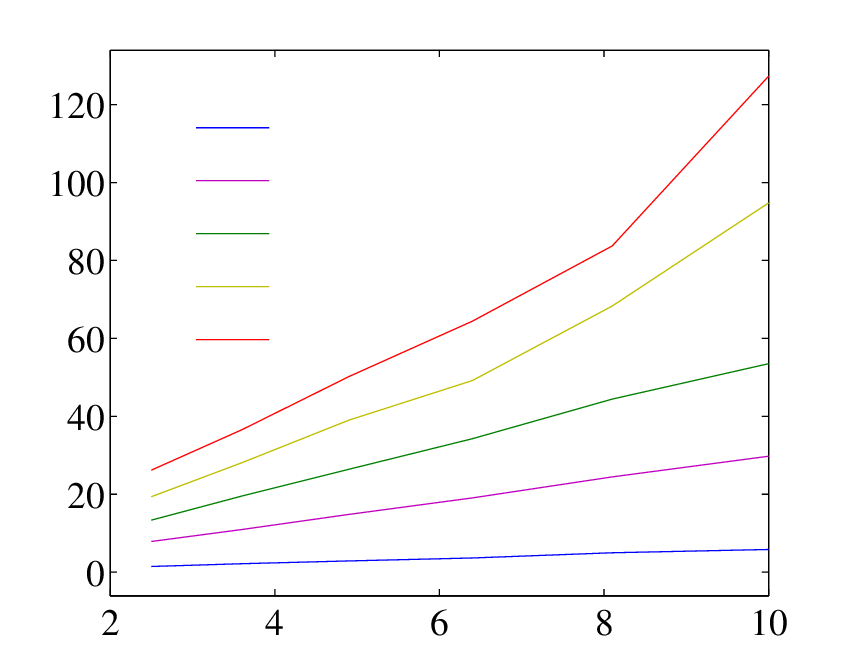}
\put(-70,110){\rotatebox{0}{ (b)}}
\put(-110,94){\rotatebox{0}{\small Max. trajectory length, $T$}}
\put(-85,85){\rotatebox{0}{\small $50$}}
\put(-85,76){\rotatebox{0}{\small $250$}}
\put(-85,67){\rotatebox{0}{\small   $450$}}
\put(-85,57){\rotatebox{0}{\small   $650$}}
\put(-85,48){\rotatebox{0}{\small   $850$}}
\put(-115,-06){\rotatebox{0}{\small \# free vars. per instant}}
\put(-30,-05){\rotatebox{0}{\small ($\times 10^3$)}}
\caption{Time to compute $\mathcal{D}_{comp}$ (within $1$\% accuracy) as a function of (a) the length of trajectories; and (b) the number of association variables per instant.}
\label{fig:profilingcodeforDcomp}
\vspace{-0.7cm}
\end{center}
\end{figure}

To interpret the plots, imagine we want to evaluate the quality of a tracker when tracking $22$ objects. Imagine that the tracker operates at $30$ frames per second and also that our tracker is noisy so that it produces a few false tracks that create approximately $10$ extra points per instant. To compute $\mathcal{D}_{comp}$, and after we extend the ground-truth and hypothesis sets from $A$ and $B$ to $A^+$ and $B^+$, we are dealing with distance matrices $D^{AB}(t)$ with about $((22+10)\times 2)^2 = 4096$ variables per instant $t$. Using Fig.~\ref{fig:profilingcodeforDcomp}
we see that it take us about $40$ seconds to evaluate the accuracy of $800/30 = 26.6$ seconds of data. If we reduce the number of free variables per frame to half, e.g. by setting $W_{ij}(t)=0$ if $D^{AB}_{ij}(t)$ is larger than a given threshold, then we can reduce the time to process $26.6$ secs. of data to about $ 20$ secs.

%
%

\section{Numerical results: Optimal tradeoff curves}
\label{sec:num_res}

To the best of our knowledge,
$\mathcal{D}_{comp}$ is the first measure that
is a metric, can be computed in
polynomial time and, in addition, deals with the tradeoff between
distance cost and switching cost optimally. In this sense,
 $\mathcal{D}_{comp}$ (and future variants) is the best metric we can hope for (cf.
Section \ref{sec:main_results}).

To illustrate this optimal tradeoff, we build and compare tradeoff curves obtained from $\mathcal{D}_{comp}$ and MOTP. A \emph{tradeoff curve} is a set of points $(swi,dist)$ where $swi$ and $dist$ are computed using \eqref{eq:swi_for_Dcomp} and \eqref{eq:dist_for_Dcomp}.
For $\mathcal{D}_{comp}$,
we obtain the different points along the curve by changing $\alpha > 0$
in \eqref{eq:Dcompwithscaling}.
For MOTP we generate the tradeoff curve by changing $thr_{MOT}$.
We do this for both synthetic and real data.
We use the Euclidean metric for $d^+$ and the component-wise
$1$-norm for the matrix norm $\|.\|$. 

The direct interpretation of our results is that $\mathcal{D}_{comp}$ is better than MOTP. However, the important underlying fact is that $\mathcal{D}_{comp}$ builds better associations than (i) the heuristics widely used in the literature, e.g. the CLEAR MOT, and (ii) the optimal associations before our work that do not allow switches, e.g. OSPA-ST. Therefore, although we restrict the comparison to MOTP, we can show similar improvement over many measures that first establish a heuristic association between $A$ and $B$. For example, Multiple Object Tracking Accuracy (MOTA), False Alarms per Frame, Ratio of Mostly Tracked trajectories, Ratio of Mostly Lost trajectories, the number of False Positives, number of False Negatives, number of ID Switches, the number of tracks Fragmentation and many of the measures listed in \cite{rothrock2000performance} and \cite{gorji2011performance}.

%
%
\vspace{-0.3cm}
\subsection{Real trackers and real data}
\label{sec:towncenternumerics}

In Figure \ref{fig:real_data_results}-(a) we show
the performance of the trackers in \cite{benfold2009guiding} and \cite{benfold2011stable} on the AVG-TownCentre data set. We call these trackers $Tracker09$ and $Tracker11$ respectively. The data set is part of the Multiple Object Tracking Benchmark \cite{motchallenge} and is widely used in computer vision. It comes from a pedestrian street filmed from an elevated point for $3$ minutes and $45$ seconds and can be downloaded from \url{http://www.robots.ox.ac.uk/ActiveVision/Research/Projects/2009bbenfold_headpose/project.html#datasets}.
In Figure \ref{fig:real_data_results}-(b) we show the performance
of the trackers in \cite{yang2012multi} and \cite{poiesi2015tracking}
on the PETS2009 data set, also part of the Multiple Object Tracking Benchmark. We call these trackers $Tracker12$ and $Tracker15$ respectively. Its duration is $1$ minute and $54$ seconds, and it can be downloaded from \url{http://www.cvg.reading.ac.uk/PETS2009/a.html}. 
More exact knowledge of these data sets and trackers is outside the scope of this paper.
The point we want to make is in regard to comparing similarity measures.
Not about comparing trackers' performance on different data sets.
We produced all the plots using the exact
same output that each tracker produced in its respective
paper, thanks to the authors who provided us with their trackers' output. When coding $d^+$, we set $M = 20$ for AVG-TownCenter and $M = 50$ for PETS2009.
\vspace{-0.0cm}
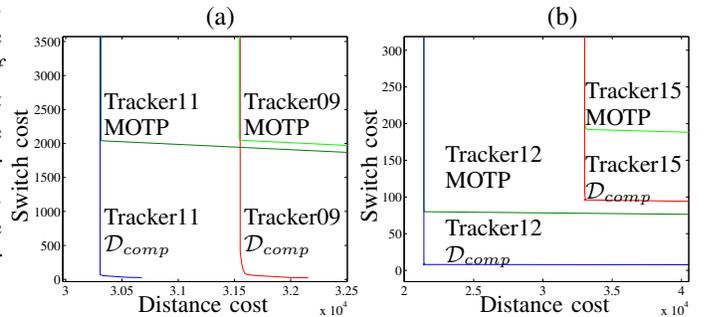
\begin{figure}[htbp] 
\begin{center}
\includegraphics[trim=0.9cm 0.cm 0cm 0cm, clip=true,height=4.0cm]{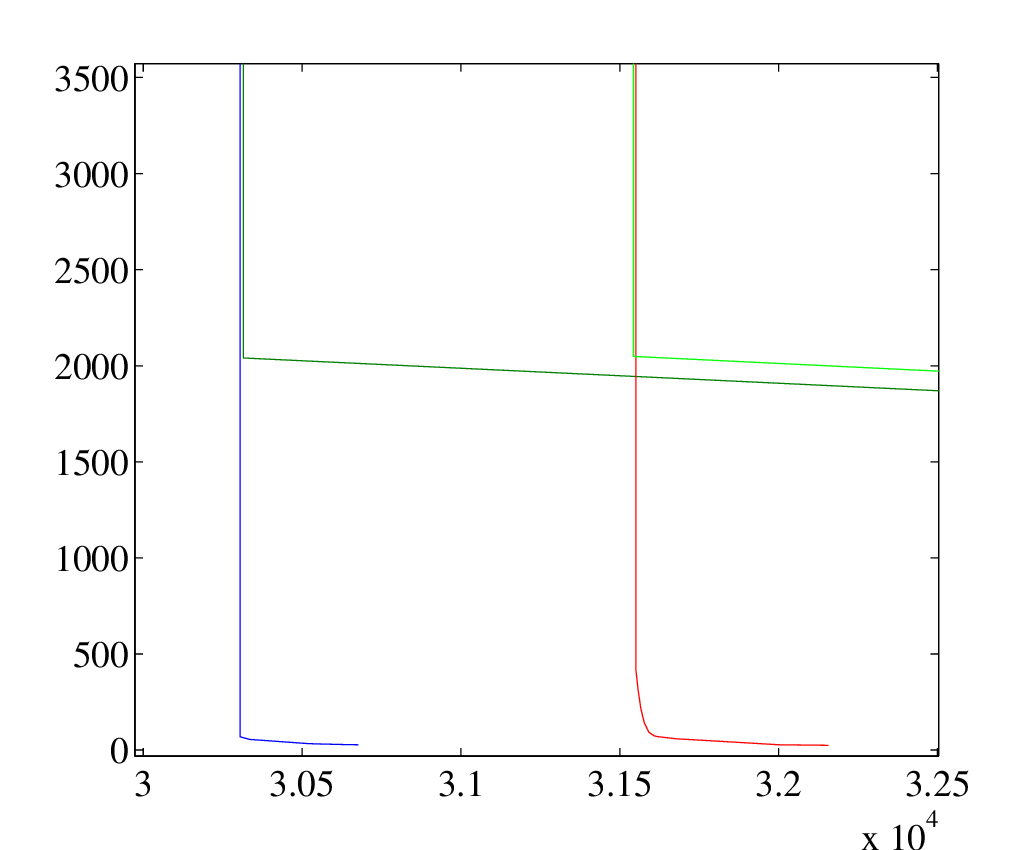}
\put(-70,110){\rotatebox{0}{ (a)}}
\put(-140,35){\rotatebox{90}{\small Switch cost}}
\put(-92,01){\rotatebox{0}{\small Distance cost}}
\put(-105,34){\rotatebox{0}{\small Tracker11}}
\put(-105,24){\rotatebox{0}{\small $\mathcal{D}_{comp}$}}
\put(-52,34){\rotatebox{0}{\small Tracker09}}
\put(-52,24){\rotatebox{0}{\small $\mathcal{D}_{comp}$}}
\put(-105,78){\rotatebox{0}{\small Tracker11}}
\put(-105,68){\rotatebox{0}{\small MOTP}}
\put(-52,78){\rotatebox{0}{\small Tracker09}}
\put(-52,68){\rotatebox{0}{\small MOTP}}
\includegraphics[trim=1.2cm 0.cm 0cm 0cm, clip=true,height=4.0cm]{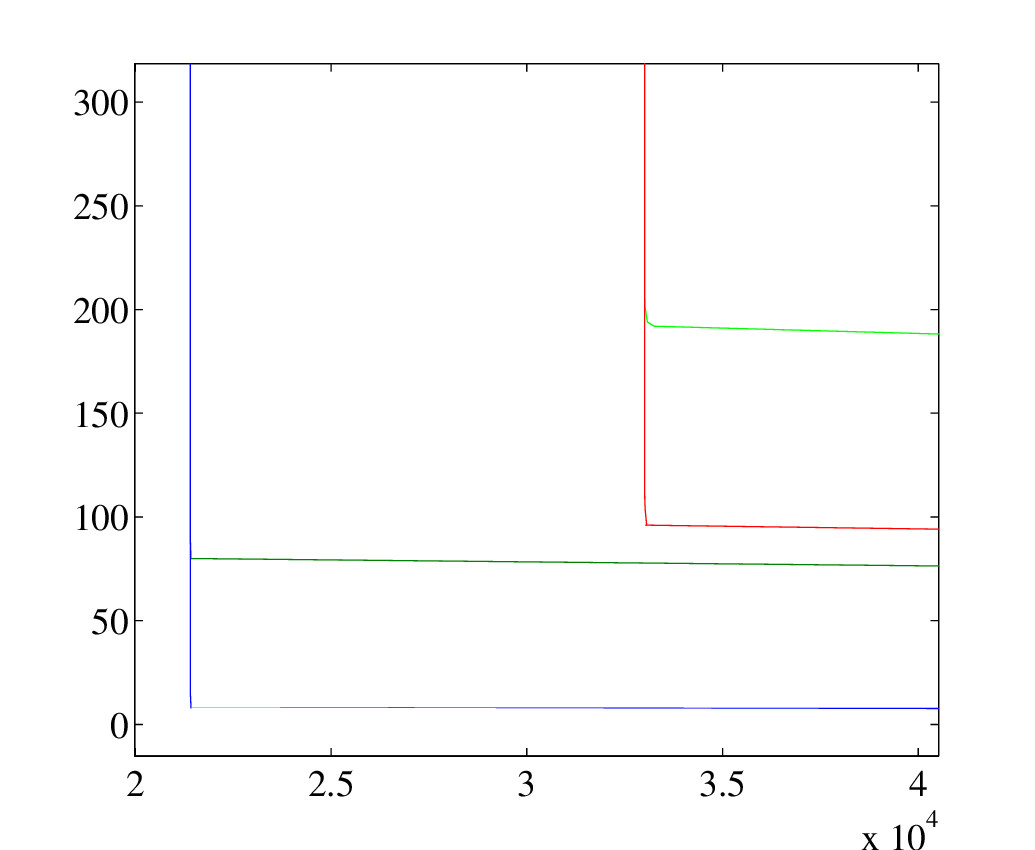}
\put(-70,110){\rotatebox{0}{ (b)}}
\put(-137,35){\rotatebox{90}{\small Switch cost}}
\put(-92,01){\rotatebox{0}{\small Distance cost}}
\put(-105,30){\rotatebox{0}{\small Tracker12}}
\put(-105,20){\rotatebox{0}{\small $\mathcal{D}_{comp}$}}
\put(-52,54){\rotatebox{0}{\small Tracker15}}
\put(-52,44){\rotatebox{0}{\small $\mathcal{D}_{comp}$}}
\put(-105,58){\rotatebox{0}{\small Tracker12}}
\put(-105,48){\rotatebox{0}{\small MOTP}}
\put(-52,82){\rotatebox{0}{\small Tracker15}}
\put(-52,72){\rotatebox{0}{\small MOTP}}
\caption{(a) Tradeoff plot for $Tracker09$ and $Tracker11$ on the AVG-TownCenter data set; and (b) tradeoff plot for $Tracker12$ and $Tracker15$ on the PETS2009 data set. Smaller values (in either axis) is better.}
\vspace{-0.4cm}
\label{fig:real_data_results}
\end{center} 
\end{figure}

As expected,
the tradeoff curves from
MOTP understate the performance of the trackers: all trackers are achieving a substantially smaller number of switches without incurring larger distance costs than what MOTP reports. Interestingly, for these trackers and data sets,
$\mathcal{D}_{comp}$ keeps the same relative ordering of performance as MOTP. It is conceivable that there are situations in which a tracker $1$ is better than a tracker $2$
according to MOTP but not according to $\mathcal{D}_{comp}$.
It would be interesting to find such examples in future work.

%
%
\vspace{-0.3cm}
\subsection{Random ensemble of trajectories}
 
Above, $A$ and $B$ are relatively close
to each other because all the trackers are good trackers.
%
%
In this section, we aim to study $A$ and $B$ that are more different. Hence, we use synthetic data to control the distance between $A$ and $B$. We randomly generate $A$ with $25$ trajectories and make
$B$ a distorted version of the trajectories in $A$. The trajectories in $A$ have random start and end times and the objects in each trajectory randomly change their velocity's direction along the way. The trajectories in $B$ are generated by randomly (i) fragmenting the trajectories in $A$, (ii) removing some of the resulting trajectories, (iii) adding noise to all trajectories and (iv) flipping or not the ID of two trajectories if they pass by each other close enough. In the end, $B$ might have more or less than $25$ trajectories. In total, we have four knobs to increase or reduce the distance between $A$ and $B$. These knobs are, (i) the amplitude of noise, $AMPnoise$; (ii) the probability of fragmenting a track at each point in time, $FRAGprob$; (iii) the probability of deleting a points in the track, $DELprob$; and (iv) the threshold distance after which we allow to tracks ID to be switched or not randomly, $SWIdist$.

Here, trajectories are far more diverse and complex than those in Section \ref{sec:towncenternumerics} and in most publicly available real data sets. Real objects, like people, have relatively simple trajectories.
In addition, we do not just test two data sets, like in Section 
\ref{sec:towncenternumerics}. We generate data for about $20$ different configurations for each of the four knobs described above and for each of these configurations we generate $30$ random sets $A$ and $B$. Hence, and in the context of computer vision tracking, it is as if we test $2400$ different
data sets of ground-truth and output trajectories.

We study the similarity between 
$A$ and $B$ using tradeoff plots with distance cost on the $y$-axis
and switching cost on the $x$-axis.
The smaller the area under the curve (AUC),
the closer $A$ and $B$ are. In Fig. \ref{fig:AUC_evol_synthetic_results} we show the average AUC for $A$ and $B$ under different knob settings. Each AUC is normalized by the largest AUC possible. The largest AUC is the product of the largest distance cost possible with the largest switching cost possible. Each point in the plots is an average over $30$ random pairs $A$ and $B$ with the same knobs' setting. In each plot we keep all but one knob constant.
\vspace{-0.0cm}
\begin{figure}[h!]
\begin{center}
\includegraphics[trim = 10mm 0mm 10mm 0mm, clip, height=4cm]{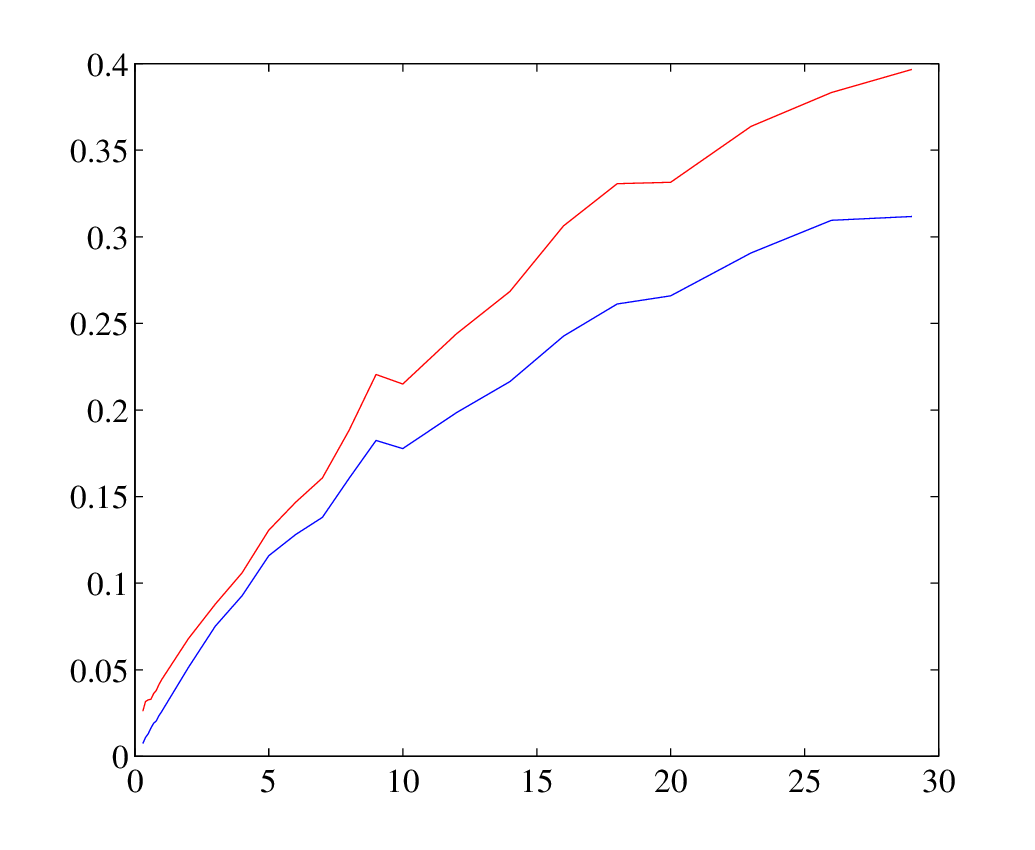}\;
\put(-132, 45){\rotatebox{90}{\small AUC}}
\put(-63  , 90){\small \color{red} MOTP}
\put(-50  , 60){\small \color{blue} $\mathcal{D}_{comp}$}
\put(-65  , 110){\small (a)}
\put(-80  , 0){\small $AMPnoise$}
\includegraphics[trim = 10mm 0mm 10mm 0mm, clip, height=4cm]{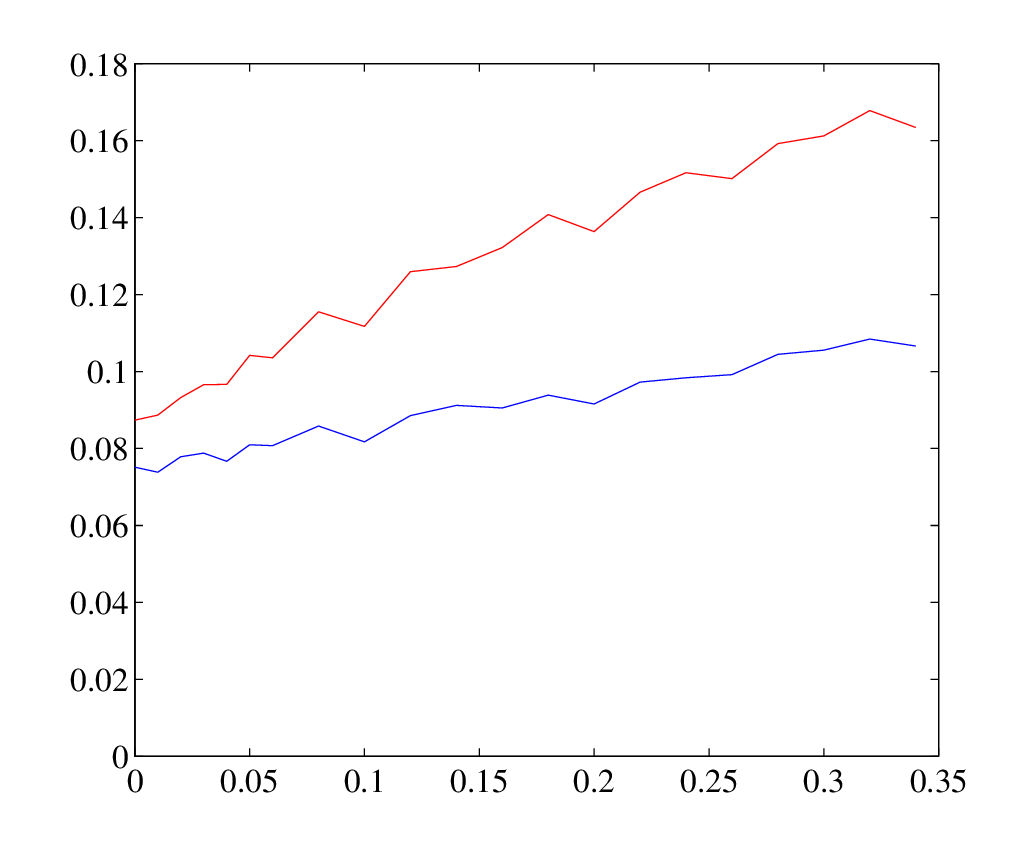}\;
\put(-131, 45){\rotatebox{90}{\small AUC}}
\put(-63  , 90){\small \color{red} MOTP}
\put(-50  , 50){\small \color{blue} $\mathcal{D}_{comp}$}
\put(-65  , 110){\small (b)}
\put(-80  , 0){\small $DELprob$}\\ \vspace{0.2cm}
\includegraphics[trim = 10mm 0mm 10mm 0mm, clip, height=4cm]{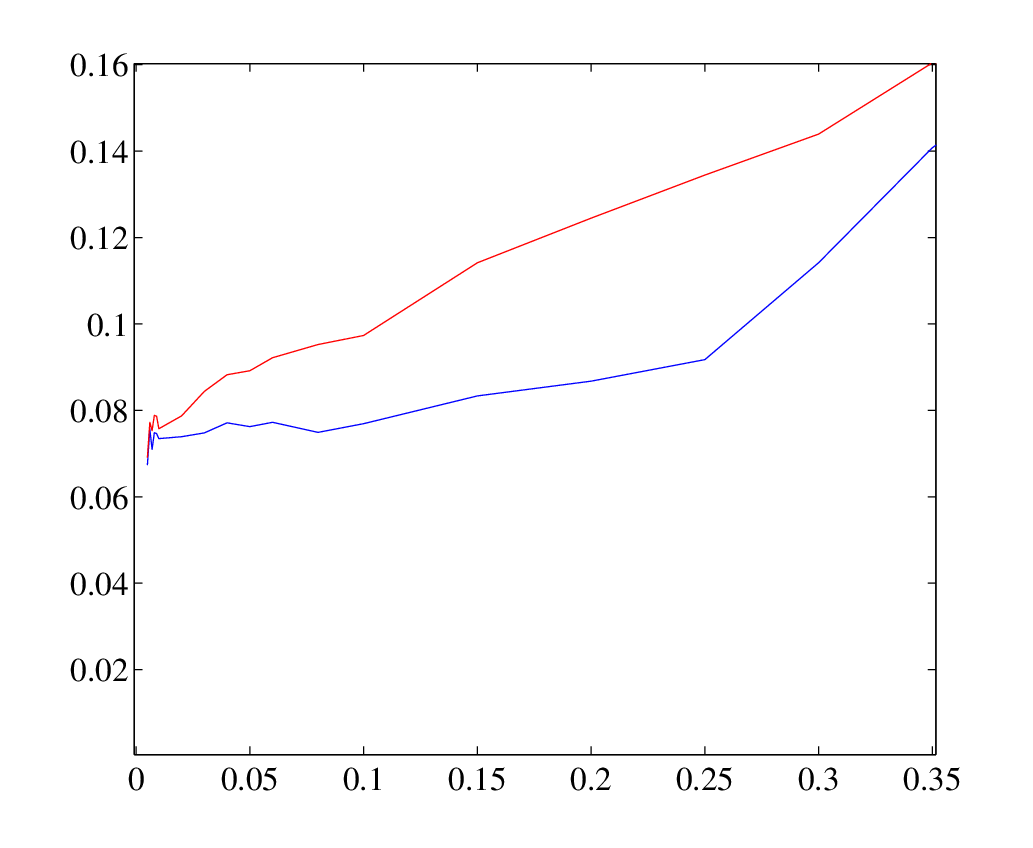}\;
\put(-132, 45){\rotatebox{90}{\small AUC}}
\put(-63  , 90){\small \color{red} MOTP}
\put(-50  , 55){\small \color{blue} $\mathcal{D}_{comp}$}
\put(-65  , 110){\small (c)}
\put(-80  , 0){\small $FRAGprob$}
\includegraphics[trim = 10mm 0mm 10mm 0mm, clip, height=4cm]{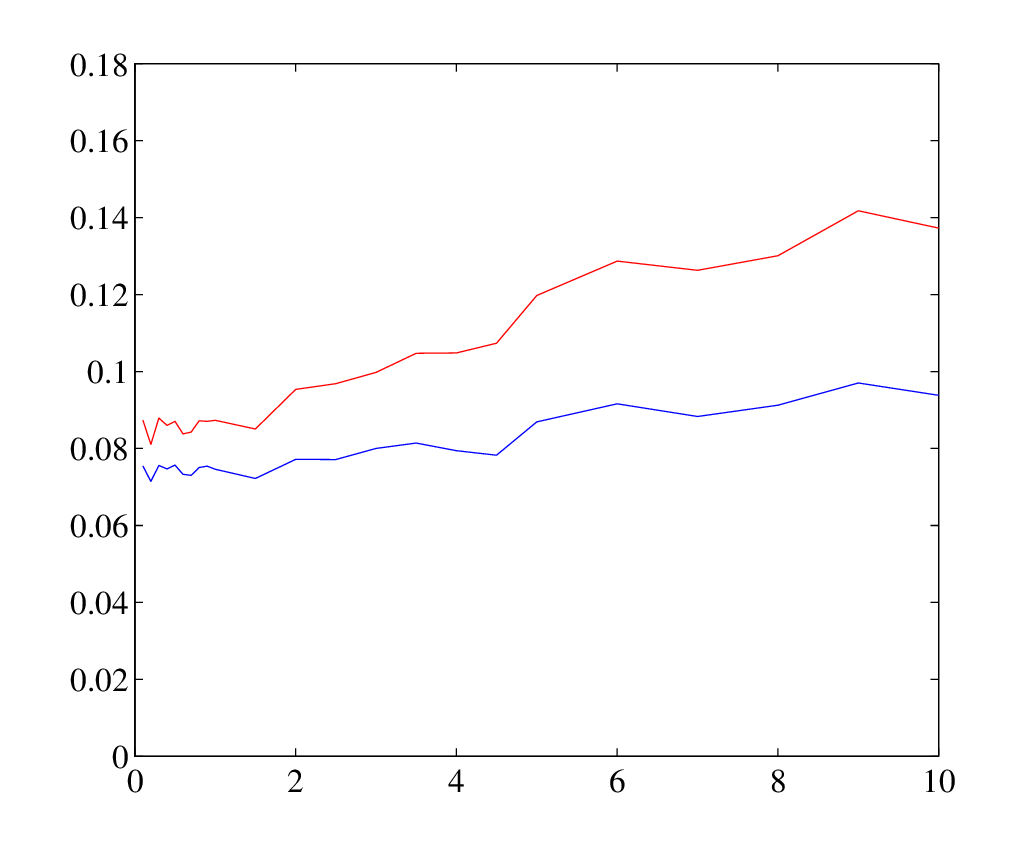}
\put(-129, 45){\rotatebox{90}{\small AUC}}
\put(-43  , 80){\small \color{red} MOTP}
\put(-50  , 48){\small \color{blue} $\mathcal{D}_{comp}$}
\put(-65  , 110){\small (d)}
\put(-80  , 0){\small $SWIdist$}
%
%
%
%
%
\caption{AUC (normalized) versus (a) noise amplitude; (b) point deletion probability; (c) fragmenting probability; and (d) switching distance. Smaller AUC is better.}
\vspace{-0.5cm}
\label{fig:AUC_evol_synthetic_results}
\end{center}
\end{figure}
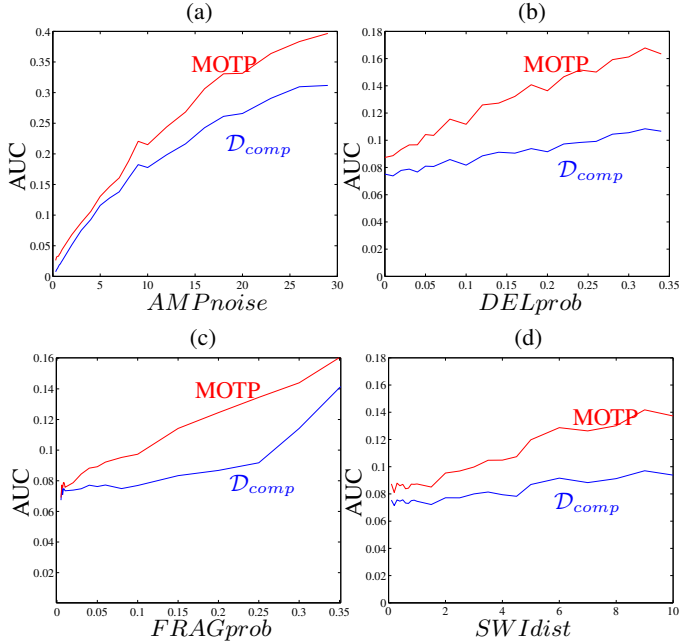

As expected from Theorem \ref{th:optimal_tradeoff}, the AUC of $\mathcal{D}_{comp}$ is smaller
than the AUC of MOTP. Note that it is incorrect to interpret these results as saying that MOTP is a scaled version of $\mathcal{D}_{comp}$. 
We are computing exactly the same quantities for both measures: $dist$ and $swi$ according to \eqref{eq:dist_for_Dcomp} and \eqref{eq:swi_for_Dcomp}. We just use a different $W$ for each measure.
The curves of Figure \ref{fig:AUC_evol_synthetic_results} have a much deeper significance: they say that the conclusions that hold for real data hold in great generality, now in $2400$ different tests and not just for two real data sets and four tackers as above. If we interpret $A$ and $B$ as the ground-truth and output of a tracker respectively, MOTP almost always says that the tracker is worse than what it really is. $\mathcal{D}_{comp}$ can see similarities between $A$ and $B$ that MOTP cannot.

%
%
\vspace{-0.3cm}
\section{Conclusion}

The problem of defining a similarity measure for sets of trajectories is crucial for computer vision, machine learning, and general AI.
An essential aspect of this problem is finding a
meaningful association between the elements of the sets.
Existing measures that define useful associations fail
to be a metric mathematically speaking, e.g., the CLEAR MOT.
The ones that are a metric,
only consider restrictive associations, i.e. associations that cannot change in time, e.g. OSPA-based metrics.
$\mathcal{D}_{comp}$ is the first that 
simultaneously (1) is a metric, (2) allows time-dependent associations and hence can associate parts of trajectories
to parts of trajectories, (3) allows controlling the complexity of
this association in a globally optimal way and (4) has polynomial computation time.

The general idea of defining mathematical metrics for sets of trajectories using
convex programs is our greatest overarching contribution. We are currently
exploring variants of our metric that allow incorporating uncertainty,
as well as richer comparisons between $A$ and $B$ without losing
its useful properties. We are also exploring the use of this metric in a machine learning application for fast classification and retrieval of human actions/activities. 
We will present these extensions in future work.

\vspace{-0.3cm}
\bibliographystyle{IEEEtran}
\bibliography{tracking_metric} 

\begin{thebibliography}{10}
\providecommand{\url}[1]{#1}
\csname url@samestyle\endcsname
\providecommand{\newblock}{\relax}
\providecommand{\bibinfo}[2]{#2}
\providecommand{\BIBentrySTDinterwordspacing}{\spaceskip=0pt\relax}
\providecommand{\BIBentryALTinterwordstretchfactor}{4}
\providecommand{\BIBentryALTinterwordspacing}{\spaceskip=\fontdimen2\font plus
\BIBentryALTinterwordstretchfactor\fontdimen3\font minus
  \fontdimen4\font\relax}
\providecommand{\BIBforeignlanguage}[2]{{%
\expandafter\ifx\csname l@#1\endcsname\relax
\typeout{** WARNING: IEEEtran.bst: No hyphenation pattern has been}%
\typeout{** loaded for the language `#1'. Using the pattern for}%
\typeout{** the default language instead.}%
\else
\language=\csname l@#1\endcsname
\fi
#2}}
\providecommand{\BIBdecl}{\relax}
\BIBdecl

\bibitem{schuhmacher2008consistent}
D.~Schuhmacher, B.-T. Vo, and B.-N. Vo, ``A consistent metric for performance
  evaluation of multi-object filters,'' \emph{Signal Processing, IEEE Trans.
  on}, vol.~56, no.~8, pp. 3447--3457, 2008.

\bibitem{keni2008evaluating}
B.~Keni and S.~Rainer, ``Evaluating multiple object tracking performance: the
  clear mot metrics,'' \emph{EURASIP Journal on Image and Video Processing},
  vol. 2008, 2008.

\bibitem{ganti1999clustering}
V.~Ganti, R.~Ramakrishnan, J.~Gehrke, A.~Powell, and J.~French, ``Clustering
  large datasets in arbitrary metric spaces,'' in \emph{Data Engineering, 1999.
  Proceedings., 15th Intern. Conf. on}.\hskip 1em plus 0.5em minus 0.4em\relax
  IEEE, 1999, pp. 502--511.

\bibitem{kleinberg2002approximation}
J.~Kleinberg and E.~Tardos, ``Approximation algorithms for classification
  problems with pairwise relationships: Metric labeling and markov random
  fields,'' \emph{Journal of the ACM (JACM)}, vol.~49, 2002.

\bibitem{yianilos1993data}
P.~N. Yianilos, ``Data structures and algorithms for nearest neighbor search in
  general metric spaces,'' in \emph{Proceedings of the fourth annual ACM-SIAM
  Symposium on Discrete algorithms}.\hskip 1em plus 0.5em minus 0.4em\relax
  Society for Industrial and Applied Mathematics, 1993, pp. 311--321.

\bibitem{blackman1986multiple}
S.~S. Blackman, ``Multiple-target tracking with radar applications,''
  \emph{Dedham, MA, Artech House, Inc., 1986, 463 p.}, vol.~1, 1986.

\bibitem{rothrock2000performance}
R.~L. Rothrock and O.~E. Drummond, ``Performance metrics for multiple-sensor
  multiple-target tracking,'' in \emph{AeroSense 2000}.\hskip 1em plus 0.5em
  minus 0.4em\relax Intern. Society for Optics and Photonics, 2000, pp.
  521--531.

\bibitem{bar2004estimation}
Y.~Bar-Shalom, X.~R. Li, and T.~Kirubarajan, \emph{Estimation with applications
  to tracking and navigation: theory algorithms and software}.\hskip 1em plus
  0.5em minus 0.4em\relax John Wiley \& Sons, 2004.

\bibitem{gorji2011performance}
A.~A. Gorji, R.~Tharmarasa, and T.~Kirubarajan, ``Performance measures for
  multiple target tracking problems,'' in \emph{Information Fusion (FUSION),
  2011 Proceedings of the 14th Intern. Conf. on}.\hskip 1em plus 0.5em minus
  0.4em\relax IEEE, 2011, pp. 1--8.

\bibitem{ristic2011metric}
B.~Ristic, B.-N. Vo, D.~Clark, and B.-T. Vo, ``A metric for performance
  evaluation of multi-target tracking algorithms,'' \emph{Signal Processing,
  IEEE Trans. on}, vol.~59, no.~7, pp. 3452--3457, 2011.

\bibitem{cvx}
I.~CVX~Research, ``{CVX}: Matlab software for disciplined convex programming,
  version 2.0,'' \url{http://cvxr.com/cvx}, Aug. 2012.

\bibitem{kuhn1955hungarian}
H.~W. Kuhn, ``The hungarian method for the assignment problem,'' \emph{Naval
  Research Logistics Quarterly}, vol.~2, no. 1-2, pp. 83--97, 1955.

\bibitem{deza2009encyclopedia}
M.~M. Deza and E.~Deza, \emph{Encyclopedia of distances}.\hskip 1em plus 0.5em
  minus 0.4em\relax Springer, 2009.

\bibitem{fujita2013metrics}
O.~Fujita, ``Metrics based on average distance between sets,'' \emph{Japan
  Journal of Industrial and Applied Mathematics}, vol.~30, 2013.

\bibitem{gardner2014measuring}
A.~Gardner, J.~Kanno, C.~A. Duncan, and R.~Selmic, ``Measuring distance between
  unordered sets of different sizes,'' in \emph{Computer Vision and Pattern
  Recognition (CVPR), 2014 IEEE Conf. on}.\hskip 1em plus 0.5em minus
  0.4em\relax IEEE, 2014.

\bibitem{hoffman2004multitarget}
J.~R. Hoffman and R.~P. Mahler, ``Multitarget miss distance via optimal
  assignment,'' \emph{Systems, Man and Cybernetics, Part A: Systems and Humans,
  IEEE Trans. on}, vol.~34, no.~3, pp. 327--336, 2004.

\bibitem{vu2014new}
T.~Vu and R.~Evans, ``A new performance metric for multiple target tracking
  based on optimal subpattern assignment,'' in \emph{Information Fusion
  (FUSION), 2014 17th Intern. Conf. on}.\hskip 1em plus 0.5em minus 0.4em\relax
  IEEE, 2014, pp. 1--8.

\bibitem{QOSPA2013}
H.~Xiaofan, R.~Tharmarasa, T.~Kirubarajan, and T.~Thayaparan, ``A track quality
  based metric for evaluating performance of multitarget filters,''
  \emph{Aerospace and Electronic Systems, IEEE Trans. on}, vol.~49, 2013.

\bibitem{nagappa2011incorporating}
S.~Nagappa, D.~E. Clark, and R.~Mahler, ``Incorporating track uncertainty into
  the ospa metric,'' in \emph{Information Fusion (FUSION), 2011 Proceedings of
  the 14th Intern. Conf. on}.\hskip 1em plus 0.5em minus 0.4em\relax IEEE,
  2011, pp. 1--8.

\bibitem{ellis2002performance}
T.~Ellis, ``Performance metrics and methods for tracking in surveillance,'' in
  \emph{Proceedings of the 3rd IEEE Intern. Workshop on Performance Evaluation
  of Tracking and Surveillance (PETS02)}.\hskip 1em plus 0.5em minus
  0.4em\relax Citeseer, 2002.

\bibitem{milan2013challenges}
A.~Milan, K.~Schindler, and S.~Roth, ``Challenges of ground truth evaluation of
  multi-target tracking,'' in \emph{Computer Vision and Pattern Recognition
  Workshops (CVPRW), 2013 IEEE Conf. on}.\hskip 1em plus 0.5em minus
  0.4em\relax IEEE, 2013.

\bibitem{fridling1991performance}
B.~E. Fridling and O.~E. Drummond, ``Performance evaluation methods for
  multiple-target-tracking algorithms,'' in \emph{Orlando'91, Orlando,
  FL}.\hskip 1em plus 0.5em minus 0.4em\relax Intern. Society for Optics and
  Photonics, 1991, pp. 371--383.

\bibitem{drummond1992ambiguities}
O.~E. Drummond and B.~E. Fridling, ``Ambiguities in evaluating performance of
  multiple target tracking algorithms,'' in \emph{Aerospace Sensing}.\hskip 1em
  plus 0.5em minus 0.4em\relax Intern. Society for Optics and Photonics, 1992,
  pp. 326--337.

\bibitem{Drummond75}
O.~E. Drummond, ``Multiple-object estimation,'' 1975.

\bibitem{colegrove1996performance}
S.~B. Colegrove, L.~Davis, and S.~J. Davey, ``Performance assessment of
  tracking systems,'' in \emph{Signal Processing and Its Applications, 1996.
  ISSPA 96., Fourth Intern. Symposium on}.\hskip 1em plus 0.5em minus
  0.4em\relax IEEE, 1996.

\bibitem{yin2007performance}
F.~Yin, D.~Makris, and S.~A. Velastin, ``Performance evaluation of object
  tracking algorithms,'' in \emph{10th IEEE Intern. Workshop on Performance
  Evaluation of Tracking and Surveillance (PETS2007)}, 2007.

\bibitem{bashir2006performance}
F.~Bashir and F.~Porikli, ``Performance evaluation of object detection and
  tracking systems,'' in \emph{IEEE Intern. Workshop on Performance Evaluation
  of Tracking and Surveillance (PETS)}, vol.~5, 2006.

\bibitem{drummond1999methodologies}
O.~E. Drummond, ``Methodologies for performance evaluation of multitarget
  multisensor tracking,'' in \emph{SPIE's Intern. Symposium on Optical Science,
  Engineering, and Instrumentation}.\hskip 1em plus 0.5em minus 0.4em\relax
  Intern. Society for Optics and Photonics, 1999, pp. 355--369.

\bibitem{pingali1996performance}
S.~Pingali and J.~Segen, ``Performance evaluation of people tracking systems,''
  in \emph{Applications of Computer Vision, 1996. WACV'96., Proceedings 3rd
  IEEE Workshop on}.\hskip 1em plus 0.5em minus 0.4em\relax IEEE, 1996, pp.
  33--38.

\bibitem{edward2009information}
K.~K. Edward, P.~D. Matthew, and B.~H. Michael, ``An information theoretic
  approach for tracker performance evaluation,'' in \emph{Computer Vision, 2009
  IEEE 12th Intern. Conf. on}.\hskip 1em plus 0.5em minus 0.4em\relax IEEE,
  2009, pp. 1523--1529.

\bibitem{porikli2004trajectory}
F.~Porikli, ``Trajectory distance metric using hidden markov model based
  representation,'' in \emph{IEEE European Conf. on Computer Vision, PETS
  Workshop}, vol.~3, 2004.

\bibitem{cayley1849lxxvii}
A.~Cayley, ``Lxxvii. note on the theory of permutations,'' \emph{The London,
  Edinburgh, and Dublin Philosophical Magazine and Journal of Science},
  vol.~34, no. 232, pp. 527--529, 1849.

\bibitem{kendall1938new}
M.~G. Kendall, ``A new measure of rank correlation,'' \emph{Biometrika}, pp.
  81--93, 1938.

\bibitem{khachiian1979polynomial}
L.~Khachiian, ``Polynomial algorithm in linear programming,'' in
  \emph{Akademiia Nauk SSSR, Doklady}, vol. 244, 1979, pp. 1093--1096.

\bibitem{boyd2004convex}
S.~Boyd and L.~Vandenberghe, \emph{Convex optimization}.\hskip 1em plus 0.5em
  minus 0.4em\relax Cambridge university press, 2004.

\bibitem{hao2016testing}
N.~Hao, A.~Oghbaee, M.~Rostami, N.~Derbinsky, and J.~Bento, ``Testing
  fine-grained parallelism for the admm on a factor-graph,'' \emph{arXiv
  preprint arXiv:1603.02526}, 2016.

\bibitem{francca2016explicit}
G.~Fran{\c{c}}a and J.~Bento, ``An explicit rate bound for over-relaxed admm,''
  in \emph{2016 IEEE International Symposium on Information Theory
  (ISIT)}.\hskip 1em plus 0.5em minus 0.4em\relax IEEE, 2016, pp. 2104--2108.

\bibitem{boyd2011distributed}
S.~Boyd, N.~Parikh, E.~Chu, B.~Peleato, and J.~Eckstein, ``Distributed
  optimization and statistical learning via the alternating direction method of
  multipliers,'' \emph{Foundations and Trends{\textregistered} in Machine
  Learning}, vol.~3, no.~1, pp. 1--122, 2011.

\bibitem{benfold2009guiding}
B.~Benfold and I.~Reid, ``Guiding visual surveillance by tracking human
  attention.'' in \emph{BMVC}, 2009, pp. 1--11.

\bibitem{benfold2011stable}
------, ``Stable multi-target tracking in real-time surveillance video,'' in
  \emph{Computer Vision and Pattern Recognition (CVPR), 2011 IEEE Conf.
  on}.\hskip 1em plus 0.5em minus 0.4em\relax IEEE, 2011, pp. 3457--3464.

\bibitem{motchallenge}
``Multiple object tracking benchmark,'' \url{http://www.motchallenge.net},
  accessed: 2015-03-01.

\bibitem{yang2012multi}
B.~Yang and R.~Nevatia, ``Multi-target tracking by online learning of
  non-linear motion patterns and robust appearance models,'' in \emph{Computer
  Vision and Pattern Recognition (CVPR), 2012 IEEE Conf. on}.\hskip 1em plus
  0.5em minus 0.4em\relax IEEE, 2012, pp. 1918--1925.

\bibitem{poiesi2015tracking}
F.~Poiesi and A.~Cavallaro, ``Tracking multiple high-density homogeneous
  targets,'' \emph{Circuits and Systems for Video Technology, IEEE Trans. on},
  vol.~25, no.~4, pp. 623--637, 2015.

\end{thebibliography}

%
%
\appendices
%
%

\section{Theory on the limitations of the CLEAR MOT}\label{app:proof_of_th_mota_bad_association}

Let us formally describe the heuristic used by the CLEAR MOT.
\begin{definition}\label{def:MOT_association}
The CLEAR MOT matching heuristic defines $\Sigma_{\text{MOT}}$ sequentially as follows.
\begin{enumerate}
\item Initialize $\Sigma_{\text{MOT}}(1)$ such that
$\sum_{i} d^+(A^+_i(1),B^+_{\Sigma_{\text{MOT}_i}(1)}(1))$
is minimal;
\item For each $t>1$ do: for all $i,j\in \{1,...,m\}$ such that
$\Sigma_{\text{MOT}_i}(t-1)=j$ and
$d^+(A^+_i(t),B^+_{j}(t)) < thr_{MOT}$ fix
$\Sigma_{\text{MOT}_i}(t)=j$. We call such matches as \emph{\bf anchored}. 
Set the non-fixed components of $\Sigma_{\text{MOT}}(t)$ such that
$\sum_{i} d^+(A^+_i(t),B^+_{\Sigma_{\text{MOT}_i}(t)}(t))$ is
minimal.
\end{enumerate}
\end{definition}

\begin{proof}[Proof of Lemma \ref{th:MOT_inconsistent}]
We construct a validating example for any $1 < thr_{MOT} < 2$, with $A = \{A_1,A_2\}$ and $B=\{B_1,B_2\}$, i.e. $m=2$, and where $A_i$ and $B_i$ are 1D trajectories. We generalize this example to any $thr_{MOT}$ and $m$ at the end.

Consider the two sets of one-dimensional trajectories $A = \{A_1,A_2 \}$ and $B = \{B_1, B_2\}$ defined in Figure \ref{fig:clearmot_bad_metric}. Time is on the $x$-axis (left to right) and space on the $y$-axis (bottom to top). $T_1$ and $T_2$
are fixed. $T_3$ grows with $T$.
%
%
%
\begin{figure}[h!]
	\begin{center}
			\includegraphics[height=5cm]{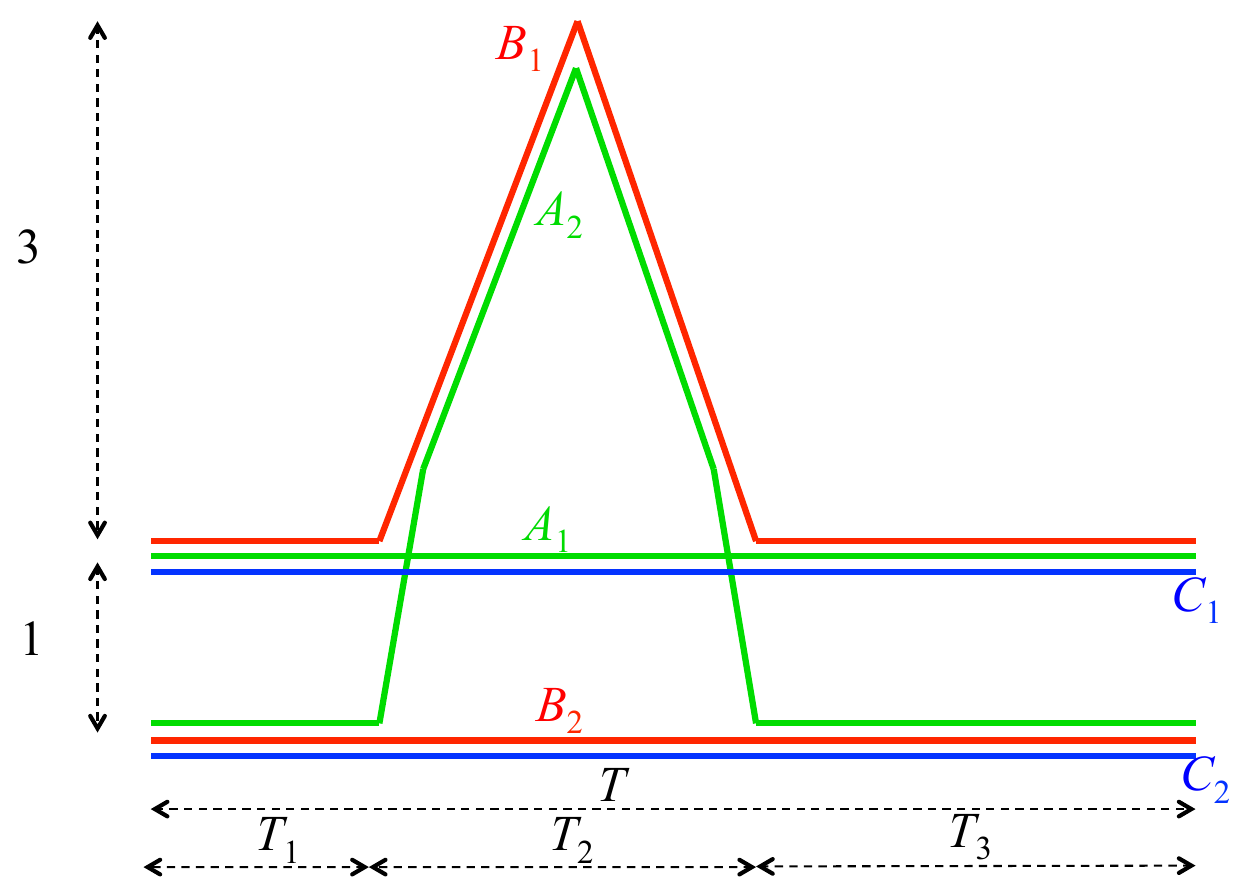}
	%
	%
		\caption{Example that shows that (a) the CLEAR MOT heuristic is bad and (b) MOTP is not a metric. For visualization purposes, trajectories that are close
		are actually on top of each other.}
		\vspace{-0.5cm}
		\label{fig:clearmot_bad_metric}
	\end{center}
\end{figure}

Since we assume $1 < thr_{MOT} <  2$, the CLEAR MOT builds an association
sequence $\Sigma_{\text{MOT}}$ that initially matches $A_1$ to $B_1$
but some time after $T_1$ matches $A_1$ to $B_2$ to minimize
the distance between matched points.
After $T_1+T_2$, this last association is anchored given that $1 < thr_{MOT}$. 
Hence,
\begin{equation*}
\Sigma_{\text{MOT}}=\{(1,2),...,(1,2),(2,1),(2,1),...\}.
\end{equation*}

The number of times that $\Sigma_{\text{MOT}}(t) \neq \Sigma_{\text{MOT}}(t+1)$ is $1$, thus $swi(\Sigma_{\text{MOT}}) = \frac{1}{T-1} = \mathcal{O}(1/T)$. 
For $t$ after $T_1+T_2$,  $\sum_{i} \|A_i(t) - B_{\Sigma_{{\text{MOT}}_i}(t)}(t)\| = 2$.
Thus $dist(\Sigma_{\text{MOT}},A,B) >  \frac{2(T-T_1-T_2)}{T} = 2 - \mathcal{O}(1/T)$. 
However, if we choose $\Sigma = \{(1,2),...,(1,2)\}$, we have $swi(\Sigma) = 0$ and
$dist(\Sigma,A,B) < \frac{T_2 \times 4}{T} = \mathcal{O}(1/T)$. The result of the lemma follows.%

To see that the proof holds for any $m$, we extend $A$ and $B$ as follows. 
Without loss of generality, we assume $m$ is even.
If the 1D trajectory $A_i$ above is equal to $(A_i(1),...,A_i(T))$, define the 2D trajectory $A^{(k)}_i$ such that
\begin{equation*}
A^{(k)}_i(t) =[A_i(t); C k] \in \mathbb{R}^2,
\end{equation*}
where $C$ is a constant large enough such that trajectories for different $k$'s are not close to each other under the $d^+$ measure. We define $B^{(k)}_i$ similarly. 
Define $$A = \{A^{(0)}_1,A^{(0)}_2,A^{(1)}_1,A^{(1)}_2,...,A^{(m-1)}_1,A^{(m-1)}_2\},$$ and similarly $B$. 
In this setting, the bounds previously computed on $swi(\Sigma_{\text{MOT}})$ and $swi(\Sigma)$ change by a factor of $m$, while the bounds on $dist(\Sigma_{\text{MOT}},A,B)$ and $dist(\Sigma,A,B)$ change by a factor of $m/2$. Thus the statement of the lemma still holds.

To extend the proof for an odd $m$, append to $A$ and $B$ two equal trajectories far away from all other trajectories such that they are matched to each other. They do not contribute to either $swi$ or $dist$. 

Finally, to see that the proof also holds for any $thr_{MOT}$, we rescale space-axis in Figure \ref{fig:clearmot_bad_metric}. This changes the bounds on $dist$ by a factor of $thr_{MOT}$ and leaves the bounds on $swi$ unchanged, thus obtains the result.
\end{proof}

%
%
\label{app:proof_of_th_motp_not_metric}

\begin{proof}[Proof of lemma \ref{th:MOPT_no_metric}]
We construct $A,B,C \in \mathcal{S}$ for which the triangle inequality is violated. Specifically, $\mathcal{D}(A,B) > \mathcal{D}(A,C) + \mathcal{D}(C,B)$.

Without loss of generality, we assume $1 < thr_{MOT} < 2$. To see that the result holds for all $thr > 0$, we employ the space-rescaling trick in the proof of Lemma \ref{th:MOT_inconsistent}.

Consider the sets $A = \{A_1, A_2\}$, $B = \{B_1, B_2\}$ and
$C = \{C_1, C_2\}$ as in Fig. \ref{fig:clearmot_bad_metric}, where the trajectories extend to some large $T$. 
 $T_1$ and $T_2$ are fixed. $T_3$ grows with $T$.
To make calculations simpler, we work with $\mathcal{D}$ divided by $T$ in Def. \ref{def:MOT_metric}.

Let us compute $\mathcal{D}(A,B)$ first. The association $\Sigma_{\text{MOT}}$ for this distance is $\{(1,2),..,(1,2),(2,1),...\}$ because (i) we start with the association $\{A_1 \leftrightarrow B_1,A_2 \leftrightarrow B_2\}$, (ii) at some point after $T_1$ we need to change the association to $\{A_1 \leftrightarrow B_2,A_2 \leftrightarrow B_1\}$
because the initial association exceeds $thr_{MOTP} < 2$ and (iii) for times after the $T_1+T_2$ the association 
$\{A_1 \leftrightarrow B_2,A_2 \leftrightarrow B_1\}$ is anchored because $thr_{MOT} > 1$. Therefore, $\mathcal{D}(A, B) > \frac{2(T-T_2)}{T}$.

Now we compute $\mathcal{D}(A,C)$. The association for $\mathcal{D}(A,C)$ is $\Sigma_{\text{MOT}} = \{(1,2),...,(1,2)\}$ because (i) we start with $\{A_1 \leftrightarrow C_1,A_2 \leftrightarrow C_2\}$, (ii) the association $A_1 \leftrightarrow C_1$ is always anchored because the distance between $A_1$ and $C_1$ is always zero and thus always smaller than $thr_{MOT} > 1$ and (iii) after some point, when the distance between $A_2$ and $C_2$ exceeds $thr_{MOT} < 2$, MOTP still keeps the association $A_2 \leftrightarrow C_2$ because $A_1$ and $C_1$ are already anchored.
Under this association, numerical computation leads to $\mathcal{D}(A, C) < \frac{4T_2}{T}$.  Similarly, $\mathcal{D}(C,B)  < \frac{4 T_2}{T}$.

Therefore, for $T$ large enough we have
$\mathcal{D}(A, B) > \frac{2(T-T_2)}{T} > 
\frac{4T_2}{T} + \frac{4T_2}{T} > \mathcal{D}(A, C) + \mathcal{D}(C, B)$.
\end{proof}

%
%

\section{Properties of our metrics: $\mathcal{D}_{nat}$} \label{app:proof_of_th_d_nat_is_metric}

To prove Theorem \ref{th:Dnaturalismetric},
we need the following lemma.

\begin{lemma} \label{th:extdismetric}
The map $d^+$ is a metric on $\mathbb{R}^p \cup \{*\}$.
\end{lemma}
\begin{proof}
We verify that $d^+$ satisfies the four conditions of a metric.
Let $x'',x',x \in \mathbb{R}^p \cup \{*\}$. 

Non-negativity and symmetry are obvious.

To verify the coincidence property, observe that $d^+(x,x') = 0$ implies either $x=x'=*$ or, since $M>0$, $d^+(x,x') = d(x,x') = 0$. Because $d$ is a metric, this in turn implies that $x=x'=0$. In other words, $d^+(x,x') = 0 \Leftrightarrow x=x'$. 

To verify the subadditivity property, we need to
consider eight different cases of the triangle inequality. We first consider the non-trivial cases.
If $x,x',x'' \in \mathbb{R}^p$, then 
\begin{align*}
d^+(x,x'') &= \min\{2M,d(x,x'') \}\\
&\leq \min \{2M,d(x,x') + d(x',x'')\} \\
&\leq \min\{2M,d(x,x')\} + \min\{2M,d(x',x'')\}\\
&= d^+(x,x') + d^+(x',x'').
\end{align*}
If $x'\hspace{-0.1cm}=\hspace{-0.1cm}*$ and $x,x'' \in \mathbb{R}^p$, then 
$M \hspace{-0.1cm}=\hspace{-0.1cm} d^+(x',x'')\hspace{-0.1cm}=\hspace{-0.1cm}d^+(x,x') $ and
\begin{align*}
d^+(x,x'') &= \min\{2M,d(x,x'') \}\leq  d^+(x,x') + d^+(x',x'').
\end{align*}
It is easy to check the other six cases.
\end{proof}

\vspace{-0.4cm}
\begin{proof}[Proof of Theorem \ref{th:Dnaturalismetric}]
Let $A,B,C$ be three elements in $\mathcal{S}$. We verify the four conditions of a metric for $\Dnat$.

\emph{Coincidence property}:
We show that $\Dnat(A,B) = 0$ if and only if $A = B$.
Recall that $A$ and $B$ are unordered sets
of trajectories. Hence $A = B$ implies that there is an isomorphism
between $A$ and $B$. In other words, they are equal apart from
a relabeling of the elements.
If $A = B$ and we set $\Sigma = (\sigma,\sigma,...,\sigma)$
then the function to minimize on the right-hand-side of equation \eqref{eq:defDnat} is equal to zero. This $\sigma$ is an isomorphism between $A$ and $B$.
Since the minimum of \eqref{eq:defDnat}
must always be non-negative,
we conclude that $A = B \Rightarrow \Dnat(A,B) = 0$.
Conversely, assume that $\Dnat(A,B) = 0$ and let $\Sigma^*=(\Sigma^*(1),...,\Sigma^*(T))$ be a
minimizer in \eqref{eq:defDnat}. $\Dnat(A,B) = 0$ implies that
$\mathcal{K}(\Sigma^*) = 0$. Therefore $\Sigma^*_i(t) = \Sigma^*_i(1)$,
for all $t$ and $i$. Since the labeling of the trajectories does not affect the computation of $\Dnat$, we assume without loss of generality that 
their labeling is such that we can write $\Sigma^*_i(t) = i$.
$\Dnat(A,B) = 0$ also implies that, for all $t$ and $i$, we have
$d^+(A^+_i(t),B^+_{\Sigma^*_i(t)}(t)) = d^+(A^+_i(t),B^+_{i}(t)) = 0$.
Since $d^+$ is a
metric, this in turn implies that $A^+_i(t) = B^+_i(t)$ for all $i$ and $t$, which is the same as saying that $A^+ = B^+$. Hence $A=B$.
To be more specific,
$A$ is equal to $B$ apart from a relabeling of its trajectories.

\emph{Symmetry property}: Since $\Dnat$ only depends on $A$ and $B$ through $d^+$ we have $\Dnat(A,B)=\Dnat(B,A)$. 

\emph{Subadditivity property}: We prove $\Dnat(A,C) \leq \Dnat(A,B) + \Dnat(B,C)$. 
First, notice that 
we can add any extra number of $*$-only
trajectories to $A$, $B$ or $C$ without changing $\Dnat$. 
Recall that $m$ is number of trajectories in $A^+$, $B^+$ and $C^+$.
In this part of the proof, $m$ should be the sum of the cardinalities of the two sets of highest cardinality among $A$, $B$ and $C$. In Section \ref{sec:setup_and_notation}, $m$ was just the sum of the cardinalities of $A$ and $B$. $T$ is the maximum time index observed in $A$, $B$ and $C$.

Since $d^+$ is a metric, we can write that
\begin{align} \label{eq:addingBtoproofDnat}
&d^+(A^+_i(t),C^+_{\Sigma_i(t)}(t)) \leq d^+(A^+_i(t),B^+_{\Sigma'_{\Sigma_i(t)}(t)}(t))\nonumber\\
& + d^+(B^+_{\Sigma'_{\Sigma_i(t)}(t)}(t),C^+_{\Sigma_i(t)}(t))
\end{align}
for any $\Sigma' = (\Sigma'(1),...,\Sigma'(T)) \in \Pi^T$ and for all $i$ and $t$. Now notice that
\begin{align*}
\sum^m_{i=1}d^+(B^+_{\Sigma'_{\Sigma_i(t)}(t)}(t),C^+_{\Sigma_i(t)}(t))
= \sum^m_{i=1}d^+(B^+_{\Sigma'_{i}(t)}(t),C^+_{i}(t)).
\end{align*}
Using this together with \eqref{eq:addingBtoproofDnat}, we can write
\begin{align*}
&\sum^T_{t=1} \sum^m_{i=1} d^+(A^+_i(t),C^+_{\Sigma_i(t)}(t)) \leq \\
&\sum^T_{t=1} \sum^m_{i=1}d^+(A^+_i(t),B^+_{\Sigma'_{\Sigma_i(t)}(t)}(t)) + d^+(B^+_{\Sigma'_{i}(t)}(t),C^+_{i}(t)).
\end{align*}
Let us define $\Sigma'' = (\Sigma''(1),...,\Sigma''(T)) \in \Pi^T$ where
$\Sigma''_i(t) = \Sigma'_{\Sigma_i(t)}(t)$. This means $\Sigma'' = \Sigma' \circ \Sigma$. We can use $\Sigma'$
to rewrite the expression above as
\begin{align} \label{eq:distanceineqproofDnat}
&\sum^T_{t=1} \sum^m_{i=1} d^+(A^+_i(t),C^+_{\Sigma_i(t)}(t))
\leq \sum^T_{t=1} \sum^m_{i=1}d^+(A^+_i(t),B^+_{\Sigma''_{i}(t)}(t))\nonumber\\
&+ \sum^T_{t=1} \sum^m_{i=1}d^+(B^+_{\Sigma'_{i}(t)}(t),C^+_{i}(t)).
\end{align}
Using Definition \ref{th:property_of_Sigma}, we can write
\begin{align} \label{eq:matchineqproofDnat}
&\mathcal{K}(\Sigma) = \mathcal{K}(\Sigma'^{-1} \circ \Sigma' \circ \Sigma) \leq \mathcal{K}(\Sigma'^{-1}) +  \mathcal{K}(\Sigma' \circ \Sigma) \nonumber\\
&= \mathcal{K}(\Sigma') +  \mathcal{K}(\Sigma'').
\end{align}
Now we add both sides of \eqref{eq:distanceineqproofDnat} and
\eqref{eq:matchineqproofDnat} and obtain
{
$
\mathcal{K}(\Sigma) + \sum^T_{t=1} \sum^m_{i=1} d^+(A^+_i(t),C^+_{\Sigma_i(t)}(t))
 \leq 
 \mathcal{K}(\Sigma'') + \sum^T_{t=1} \sum^m_{i=1}d^+(A^+_i(t),B^+_{\Sigma''_{i}(t)}(t))
 + \mathcal{K}(\Sigma')+\sum^T_{t=1} \sum^m_{i=1}d^+(B^+_{\Sigma'_{i}(t)}(t),C^+_{i}(t))$.
}

Finally, we find the minimum of both sides of the inequality over all pairs
of $\Sigma$ and $\Sigma'$. Recall the relationship $\Sigma'' = \Sigma' \circ \Sigma$ and the fact that we can choose $\Sigma'$ independently of $\Sigma$. Consequently,
\begin{align*} 
&\min_{\Sigma \in \Pi^T} \mathcal{K}(\Sigma) + \sum^T_{t=1} \sum^m_{i=1} d^+(A^+_i(t),C^+_{\Sigma_i(t)}(t))\\
& \leq \min_{\Sigma''\in \Pi^T} \mathcal{K}(\Sigma'') + \sum^T_{t=1} \sum^m_{i=1}d^+(A^+_i(t),B^+_{\Sigma''_{i}(t)}(t))\\
& + \min_{\Sigma'\in \Pi^T} \mathcal{K}(\Sigma')+\sum^T_{t=1} \sum^m_{i=1}d^+(B^+_{\Sigma'_{i}(t)}(t),C^+_{i}(t)).
\end{align*}
Hence, $\Dnat(A,C) \leq \Dnat(A,B) + \Dnat(B,C)$.
\end{proof}

%
%

\label{app:proof_for_K_propreties}

\begin{proof}[Proof of Theorem \ref{th:K_count}]
The first two properties of the permutation measure are trivial to verify. To verify the third property, it is sufficient to prove that 
\begin{align*}
\mathbb{I}\left((\Sigma'(t+1) \circ \Sigma(t+1)) \circ (\Sigma'(t) \circ \Sigma(t))^{-1} \neq I\right)\\
\leq \mathbb{I}(\Sigma(t+1) \circ \Sigma(t)^{-1} \neq I) +\mathbb{I}(\Sigma'(t+1) \circ \Sigma'(t)^{-1} \neq I).
\end{align*}
Since the left-hand-side is at most $1$, we only need to consider the case in which the right-hand-side is less than $1$, i.e.,
$$\mathbb{I}(\Sigma(t+1) \circ \Sigma(t)^{-1} \neq I) +\mathbb{I}(\Sigma'(t+1) \circ \Sigma'(t)^{-1} \neq I)=0.$$
In this case, it follows that $\Sigma(t+1) = \Sigma(t)$ and $\Sigma'(t+1) = \Sigma'(t)$. Consequently, the left-hand-side equals $0$ as well in this case. Hence the property is verified.
\end{proof}
\begin{proof}[Proof of Theorem \ref{th:K_trans}]
The first two properties of the permutation measure
are immediate to check from the definition of $\mathcal{K}_{trans}$. To prove the third property it suffices to show that
\begin{multline*}
k_{Cayley}\left((\Sigma'(t+1) \circ \Sigma(t+1)) \circ (\Sigma'(t) \circ \Sigma(t))^{-1}\right)\\
\leq k_{Cayley}(\Sigma(t+1)) \circ  \Sigma(t)^{-1}) + k_{Cayley}(\Sigma'(t+1) \circ \Sigma'(t))^{-1}).
\end{multline*}
To prove this we use of the following two facts.

Fact 1: if $\sigma,\sigma' \in \Pi$ then $k_{Cayley}(\sigma\circ\sigma') \leq k_{Cayley}(\sigma) + k_{Cayley}(\sigma')$.
To see why this is true, notice that this inequality can be rewritten as 
$k_{Cayley}(a \circ c^{-1}) \leq k_{Cayley}(a \circ b^{-1}) + k_{Cayley}(b \circ c^{-1})$ where
$a = \sigma\circ\sigma'$, $b = \sigma'$ and $c = I$, the identity permutation. Then, it is basically saying that shortest set of transpositions that take $\sigma_1$ to $\sigma_3$ is smaller than any set of transpositions that first take $\sigma_1$ to $\sigma_2$ and then $\sigma_2$ to $\sigma_3$.

Fact 2: if $\sigma,\sigma' \in \Pi$ then $k_{Cayley}(\sigma\circ \sigma' \circ\sigma^{-1}) = k_{Cayley}(\sigma' )$. This is true because $\sigma\circ \sigma' \circ\sigma^{-1}$ is just a relabeling of the permutation  $\sigma'$ and $k_{Cayley}$ is invariant to relabeling.

Let 
\vspace{-0.3cm}
\begin{align*}
&\sigma_A=\Sigma(t)\circ \Sigma(t+1)^{-1}, \;\;\;\;\; \sigma_B = \Sigma'(t)\circ \Sigma'(t+1)^{-1},\\
&\text{and }\sigma_C=\Sigma'(t)\circ \sigma_A\circ \Sigma'(t)^{-1}.
\end{align*}
Observe that the permutation $\sigma_C \circ \sigma_B$ satisfies
$$\sigma_C \circ \sigma_B \circ \left((\Sigma'(t+1) \circ \Sigma(t+1)) \circ (\Sigma'(t) \circ \Sigma(t))^{-1}\right) = I.$$
Finally, we can apply Fact 2 followed by Fact 1, and obtain (dropping $_{Cayley}$ for clarity).
\begin{align*}
 &k\left((\Sigma'(t+1) \circ \Sigma(t+1)) \circ (\Sigma'(t) \circ \Sigma(t))^{-1}\right)\\
&= k(\sigma^{-1}_B \circ \sigma^{-1}_C)\leq k(\sigma^{-1}_B) + k( \sigma^{-1}_C)= k(\sigma^{-1}_B) + k( \sigma^{-1}_A)\\
&= k(\Sigma(t+1)\circ\Sigma(t)^{-1}) + k(\Sigma'(t+1)\circ\Sigma'(t)^{-1}).
\end{align*}
\vspace{-0.3cm}
\end{proof}
%

%
%

\label{app:necessary_conditions_for_D_nat_metric}

\vspace{-0.3cm}
\begin{proof}[Proof of Theorem \ref{th:K_maxcount_not_proper}]
We give a counter example that violates property (iii)
for $\beta = 1$. It is also easy to come up with similar
counter examples that violate property (iii) for any value
of $\beta \geq 1$. Let $I = (1,2)$ be
the identity permutation and let $\sigma = (2,1)$ be the permutation that swaps $1$ and $2$. Let $\Sigma = (I,\sigma,\sigma)$ and let $\Sigma' = (I,I,\sigma)$.
We have $\mathcal{K}_{maxcount}(\Sigma) = 1$ and 
$\mathcal{K}_{maxcount}(\Sigma') = 1$ but
$\mathcal{K}_{maxcount}(\Sigma' \circ \Sigma) = \infty > \mathcal{K}_{maxcount}(\Sigma') + \mathcal{K}_{maxcount}(\Sigma)$.
\end{proof}
%

\begin{proof}[Proof of Theorem \ref{th:K_maxcount_leads_to_not_metric}]
We provide a special case where the triangle inequality is violated.
Let $A$, $B$ and $C$ be three sets of trajectories. We assume $\beta = 1$,
but it is easy to change $A$, $B$ and $C$ such that
the proof holds for any $\beta$,\footnote{It is easy to see that
we do not in fact need $d$ to be the Euclidean distance for the proof to hold.}.
Let $I = (1,2)$ be
the identity permutation and let $\sigma^0 = (2,1)$
be the permutation that swaps $1$ and $2$.
Let $A = \{A_1,A_2\}$, $B = \{B_1,B_2\}$, $C = \{C_1,C_2\}$, where
\begin{align*}
&A_1 = (2, -2, -2),\;\; B_1 = (2,2,2), \;\;\;\;\;\;\;\;\;\;C_1 = (2, 2, -2),\\
&A_2 = (-2, 2, 2), \;\;\;\;\,B_2 = (-2,-2,-2), \;\;C_2 = (-2, -2, 2).
\end{align*}
Now consider equation \eqref{eq:defDnat} with $\mathcal{K} = \mathcal{K}_{maxcount}$ and $\beta = 1$.
The optimization problem for $\mathcal{D}_{nat}(A,B)$,
\begin{align*}
\min_{\Sigma \in \Pi^T} \Big\{ \mathcal{K}(\Sigma)
+\sum^3_{t=1} \sum^{2}_{i=1}d^+(A^+_i(t),B^+_{\Sigma_i(t)}(t)) \Big\}
\end{align*}
has a minimum of $1$ at
$\Sigma = (I,\sigma^0,\sigma^0)$. We can see this because $\mathcal{K}$ forces us to either
do one switch only or no switch at all. Doing no switch at all makes us
incur a cost larger than $1$ in the distance term.  
Similarly, the optimization problem for $\mathcal{D}_{nat}(B,C)$
has a minimum of $1$ at
$\Sigma = (I,I,\sigma^0)$. When we solve the optimization
problem for $\mathcal{D}_{nat}(A,C)$, we are only allowed
to perform one change in the association between $A$
and $C$, otherwise the term $\mathcal{K}_{maxcount}$
makes us pay a very large cost. With only one change
in the association, we incur a distance of $4$
for $t = 1$ or $t = 2$ or $t = 3$. Hence, $\mathcal{D}_{nat}(A,C) \geq 4 > 1 + 1 = \mathcal{D}_{nat}(A,B) + \mathcal{D}_{nat}(B,C)$.
\end{proof}
%

%
%
%
\section{Properties of our metrics: $\mathcal{D}_{comp}$} \label{app:proof_that_D_comp_is_metric}

\begin{proof}[Proof of Theorem \ref{th:D_comp_is_metric}]
Let $A, B$ and $C$ be any elements of $\mathcal{S}$.\\
\emph{Coincidence property}: If $\mathcal{D}_{comp}(A,B) = 0$,
then $W(t) = W(1)$ for all $t$. Hence, $D^{AB}_{ij}(t) = 0$ for all $i,j$ such that
$W_{ij}(1) > 0$. Recall that if $W_{ij}$ is a doubly stochastic matrix, then there exists a permutation $\sigma = \left(\sigma_1,\dots,\sigma_m\right) \in \Pi$, such that
$W_{i\sigma_i} > 0$ for all $i$.  Therefore, $D^{AB}_{i\sigma_i}(t) = 0$ for all $t$ and $i$. In other words,
$A$ and $B$ are identical apart from a relabeling of their elements.\\
\emph{Symmetry property}: Using the properties of {\bf trace}, we have
\begin{align*}
&\text{\bf tr}(W^{\dagger}(t)D^{AB}(t)) =
\text{\bf tr}((W^{\dagger}(t)D^{AB}(t))^{\dagger} )\\
&= \text{\bf tr}(D^{AB}(t)^{\dagger}W(t))= \text{\bf tr}(D^{BA}(t)W(t))
= \text{\bf tr}(W(t) D^{BA}(t)). 
\end{align*}

Minimizing with respect to $\{W(t)\}$ is the same as minimizing with
respect to $\{W(t)^{\dagger}\}$. Thus $\mathcal{D}_{comp}(A,B) = \mathcal{D}_{comp}(B,A)$.\\
\emph{Subadditivity property}: We prove that
$\mathcal{D}_{comp}(A,C) \leq \mathcal{D}_{comp}(A,B) + \mathcal{D}_{comp}(B,C)$. 

First, notice that 
we can add any extra number of $*$-only
trajectories to $A$, $B$ or $C$ without changing $\mathcal{D}_{comp}$. 
Recall that $m$ is number of trajectories in $A^+$, $B^+$ and $C^+$.
In this part of the proof, $m$ should be the sum of the cardinalities of the two sets of highest cardinality among $A$, $B$ and $C$. In Section \ref{sec:setup_and_notation}, $m$ was just the sum of the cardinalities of $A$ and $B$. $T$ is the maximum time index observed in $A$, $B$ and $C$.

First note that, since $d^+$ is a metric we have
that $D^{AC}(t)_{ij} \leq D^{AB}_{ik}(t) + D^{BC}_{kj}(t)$ for any $k$.
Let $W_1 = (W_1(1),...,W_1(T))  \in \mathcal{P}^T$ and
$W_2 = (W_2(1),...,W_2(T)) \in  \mathcal{P}^T$. We multiply both
sides of the previous inequality by $W_1(t)_{ik}W_2(t)_{kj}$ and
sum over $i,j,k$ to obtain
\vspace{-0.2cm}
\begin{multline*}
\sum^{m}_{i,j,k=1} {W_1}_{ik}{W_2}_{kj} {D^{AC}}_{ij}\leq\\ \sum^{m}_{i,j,k=1} {W_1}_{ik}{W_2}_{kj} D^{AB}_{ik}
+ {W_1}_{ik}{W_2}_{kj}D^{BC}_{kj},
\end{multline*}
%
%
where we omitted $t$ in $W$ and $D$.
Since $W_1(t), W_2(t) \in \mathcal{P}$, we have $\sum^m_{j=1} W_2(t)_{kj} = 1$ and
$\sum^m_{i=1} W_1(t)_{ik} = 1$. To simplify this expression, and again omitting $t$ in $W$ and $D$, we re-write the previous inequality in matrix notation as
\begin{multline*}
\text{\bf tr}\left(({W_1}{W_2})^{\dagger} D^{AC}\right) \leq
\text{\bf tr}({W_1}^{\dagger} D^{AB}) + 
\text{\bf tr}({W_2}^{\dagger} D^{BC}).
\end{multline*}
From our assumption on $\|.\|$ (c.f. Property \eqref{eq:propofnormforDcomp}), we have
\begin{multline*}
\|W_1(t+1)W_2(t+1) - W_1(t)W_2(t)\| \leq \|W_1(t+1) - W_1(t)\| 
\\+ \|W_2(t+1) - W_2(t)\|.
\end{multline*}
Adding the last two inequalities, summing over $t$ and minimizing over $W_1, W_2 \in \mathcal{P}^T$ 
we have
{\small
\begin{align*}
&\min_{W_1,W_2 \in \mathcal{P}^T} \sum^{T-1}_{t=1} \|W_1(t+1)W_2(t+1) \\
&- W_1(t)W_2(t) \| +
\sum^T_{t=1} \text{\bf tr}\left((W_1(t)W_2(t))^{\dagger} D^{AC}(t)\right)\\
&\leq \min_{W_1 \in \mathcal{P}^T}
\sum^{T-1}_{t=1} \|W_1(t+1)- W_1(t) \|+
\sum^T_{t=1} \text{\bf tr}(W_1(t)^{\dagger} D^{AB}(t)) \\
&+\min_{W_2 \in \mathcal{P}^T}
\sum^{T-1}_{t=1} \|W_2(t+1)- W_2(t) \|
+
\sum^T_{t=1} \text{\bf tr}(W_2(t)^{\dagger} D^{BC}(t)).
\end{align*}
}
Note that for any $W_1(t), W_2(t) \in \mathcal{P}$, we have
$W_1(t) W_2(t) \in \mathcal{P}$.
Therefore, in the minimization performed on the left hand side,
we can replace $W_1(t) W_2(t)$ by a single $W(t)$.
The subadditivity property follows.
\end{proof}
%

%
%

\label{app:proof_that_norms_for_Dcomp_are_many}

\vspace{-0.3cm}
\begin{proof}[Proof of Lemma \ref{th:many_norms_satisfy_D_comp}]
Let $w_1,w_2,w'_1, w'_2 \in \mathcal{P}$, whose norms are less than $1$.
\begin{multline*}
\|w'_2w_2 - w'_1 w_1\| \leq \|w'_2(w_2 -w_1 + w_1)- w'_1 w_1\| \\
\leq \|w'_2(w_2 -w_1)\| + \| w'_2 w_1 - w'_1 w_1\|
\leq \|w'_2\| \|(w_2 -w_1)\| \\+ \| (w'_2 -w'_1)\| \|w_1\|
\leq \|(w_2 -w_1)\| + \| (w'_2 -w'_1)\|.
\end{multline*}
The third inequality above is obtained using sub-multiplicity.
\end{proof}

%
%

\label{app:proof_that_D_comp_equals_LP}

\vspace{-0.3cm}
\begin{proof}[Proof of Theorem \ref{th:dcompisLP}]
The constraints which define $\mathcal{P}$
are a set of linear constraints. In addition, the first term
in the objective is a linear function of $\{W(t)\}$.
Notice also that we can replace
the term $\sum^{T-1}_{t=1} \|W(t+1) - W(t)\|$ in the objective
by $\sum^{T-1}_{t=1} e_t$,
if we add the additional constraints
that $\|W(t+1) - W(t)\| \leq e_t$ for all $t$.
We can also represent these additional constraints
as linear constraints.
Specifically, since \vspace{-0.1cm}$$\|W(t+1) - W(t)\| = \max_{j} \sum_i |W_{ij}(t+1) - W_{ij}(t)|,$$
each of these constraints can be replaced by 
$\sum_i h_{ij}(t) \leq e_t$ for all $t$ and $j$ 
if we add the additional constraints
$$|W_{ij}(t+1) - W_{ij}(t)| \leq h_{ij}(t)$$ for all $i$, $j$ and $t$.
Each of these constraints can be replaced two linear constraints,
namely 
\begin{align*}
W_{ij}(t+1) - W_{ij}(t)) \leq h_{ij}(t),\\
-W_{ij}(t+1) + W_{ij}(t)) \leq h_{ij}(t)
\end{align*}

%

\end{proof}
%

%
%

\clearpage
\section{Some extra observations}  \label{app:extra}

In this section make some extra observations that are not essential
to understand what was presented above. We use the same setup and notation as explained in Section \ref{sec:setup_and_notation} of the main paper. 

%
%

\subsection*{Another example of OSPA-ST producing counter-intuitive results}

Consider the setup in Figure \ref{fig:ospa_bad_example}
where two people, $A_1$ and $A_2$, pass by each other
and two trackers, $B$ and $C$, produce output trajectories
$B_1$, $B_2$, $C_1$ and $C_2$. Space is measure on the $y$-axis (bottom
to top) and time on the $x$-axis (left to right). To aid visualization, close lines should be considered on top of each other.

\begin{figure}[h]
\begin{center}
\includegraphics[height=6.2cm]{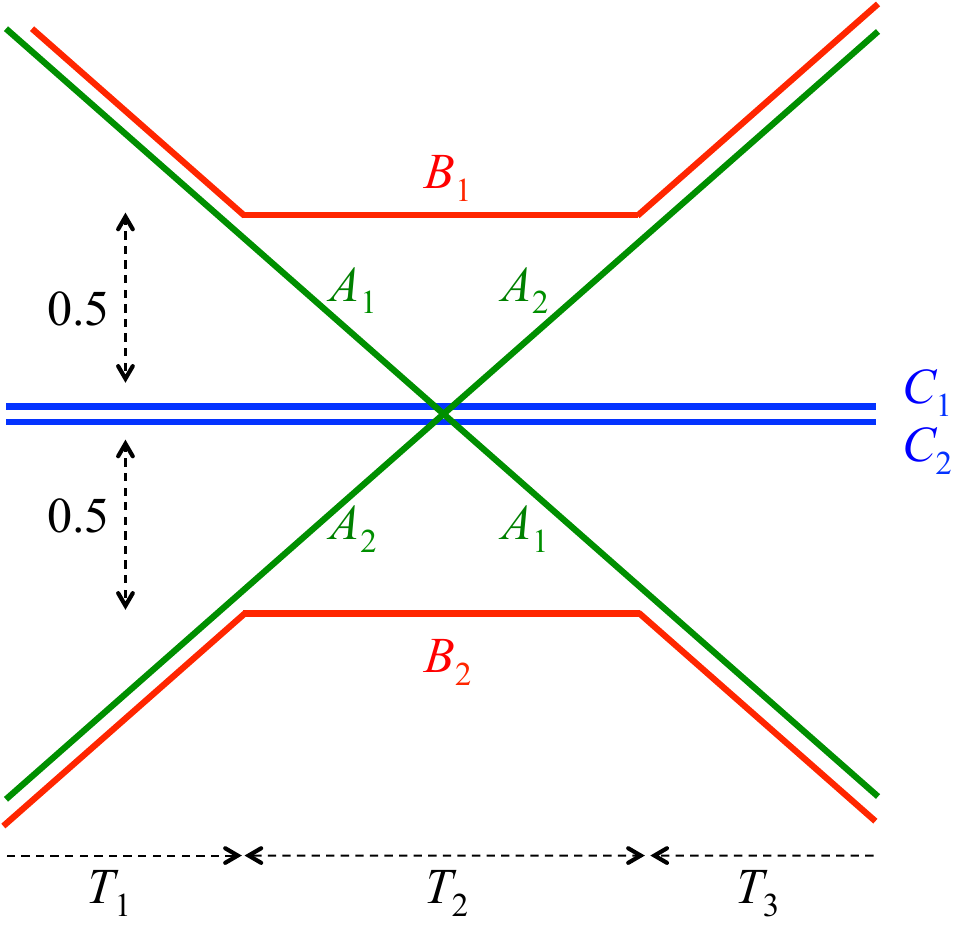}
\vspace{-0.3cm}
\caption{Counter example that shows that $\mathcal{D}_{MOTA}$ is not a metric.}
\label{fig:ospa_bad_example}
\end{center}
\end{figure}

Let $T_1 = T_3$.
By symmetry and the fact that OSPA-ST does not allow associations that change with time we can assume without loss of generality that in computing $\mathcal{D}_{OSPA-ST}(A,B)$ OSPA-ST associates $A_1$ to $B_1$ and $A_2$ to $B_2$ and that in computing $\mathcal{D}_{OSPA-ST}(A,C)$ OSPA-ST associates $A_1$ to $C_1$ and $A_2$ to $C_2$. 

If $T_2 << T_1 = T_3$, the average distance between $A$ and $B$ can be arbitrarily close to the average distance between $A$ and $C$. More formally,
$\mathcal{D}_{OSPA-ST}(A,C) / \mathcal{D}_{OSPA-ST}(A,B) \rightarrow 1$ as $T_1=T_3 \rightarrow \infty$. In other words, OSPA-ST says that tracker $C$ is
as good as tracker $B$ while our intuition says
that $B$ is the better tracker. $B$ produces good tracks for most of the time
and simply makes an identity switch during internval $T_2$. On the other hand,
$C$ is a trivial tracker that always outputs $0$ regardless of the input.

\subsection*{MOTA does not define a metric}

First recall how to compute MOTA between two sets of trajectories
$A$ and $B$. First compute the CLEAR MOT association $\Sigma_{MOT}$ between $A$ and $B$ as described in Appendix \ref{app:proof_of_th_mota_bad_association}.
Second, compute a positive linear combination, $\mu$, of three quantities $\nu_1$, $\nu_2$ and $\nu_3$ where: $\nu_1$ is the number of changes of association between trajectories; $\nu_2$ is the number of points of trajectories in $A$ unassociated to any point in $B$; and $\nu_3$ is the number of points of trajectories in $B$ unassociated to any point in $A$.
MOTA is defined as $1-\mu$ but since distances must decrease as $A$
and $B$ get similar we define $\mathcal{D}_{MOTA}(A,B) = \mu$.

\begin{lemma}\label{th:MOPT_no_metric}
$\mathcal{D}_{MOTA}$ is not a metric for any threshold $thr_{MOT} \hspace{-0.1cm}>\hspace{-0.1cm} 0$ or positive linear combination used to
define $\mu$.
\end{lemma}

\begin{proof}

We given an example of three sets
$A$, $B$ and $C$ such $\mathcal{D}_{MOTA}(A,B) > 0$ while
$\mathcal{D}_{MOTA}(A,C) = \mathcal{D}_{MOTA}(C,B) = 0$, hence the triangle inequality is violated. Our counter example works for $ 0.5 < thr_{MOT} < 1$ but by scaling
space we can prove the lemma for any $thr_{MOT}$.

Consider $A$, $B$ and $C$ as in Figure \ref{fig:toy_example} in the main paper, which we reproduce here
for convenience. To aid visualization, close lines should be considered on
top of each other and transitions in $A$ almost instantaneous. Time is on the
$x$-axis (left to right) and space is on the $y$-axis (bottom to top).
\begin{figure}[h]
\begin{center}
\includegraphics[height=3.8cm]{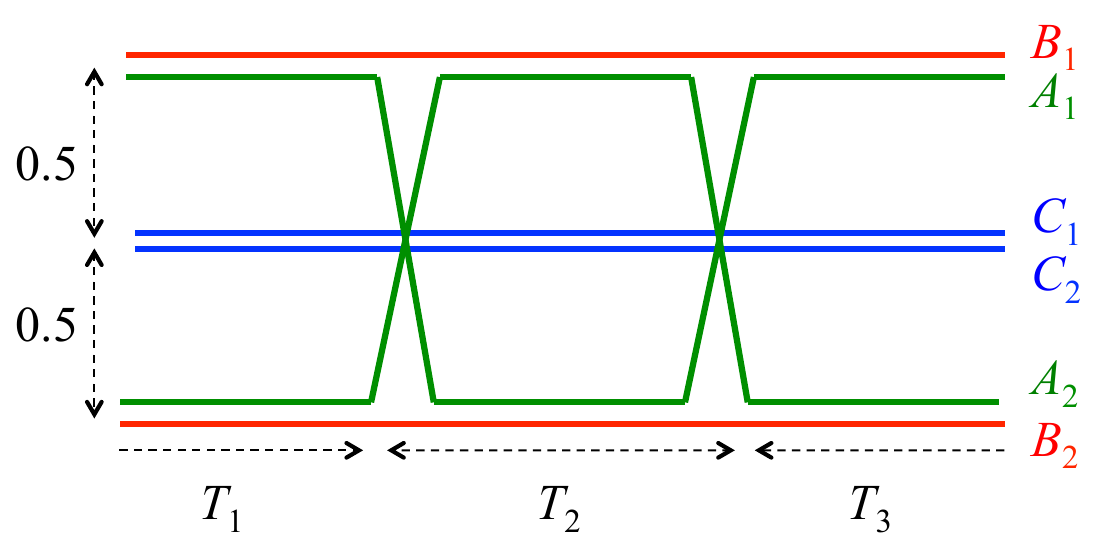}
\vspace{-0.3cm}
\caption{Counter example that shows that $\mathcal{D}_{MOTA}$ is not a metric.}
\end{center}
\end{figure}

When we compute $\mathcal{D}_{MOTA}(A,C)$,  $\Sigma_{MOT}$
associates $A_1$ to $C_1$ and  $A_2$ to $C_2$ for all times $t$.
(we can interchange $C_1$ and $C_2$ since they are equal).
There are no changes in association
because the distances computed are always less than $0.5$
and $0.5 < thr_{MOT}$, thus $\nu_1 = 0$.
In addition,  there are no unallocated tracks so $\nu_2 = \nu_3 = 0$.
Therefore $\mathcal{D}_{MOTA}(A,C) = \mu = 0$ regardless of the coefficients in the linear combination $\mu$.
Similarly we get $\mathcal{D}_{MOTA}(C,B) = 0$.

When we compute $\mathcal{D}_{MOTA}(A,B)$
we find the following: before $T_1$, $\Sigma_{MOT}$ associates
$A_1$ to $B_1$; after $T_1$ and before $T_1+T_2$, the distance
between $A_1$ and $B_1$ becomes $1 > thr_{MOT}$ and so
$\Sigma_{MOT}$ changes association and matches $A_1$ to $B_2$.
After $T_1+T_2$, the distance between $A_1$ and $B_2$
becomes $1 > thr_{MOT}$ and so $\Sigma_{MOT}$ changes association
back to $A_1$ matches to $B_1$. We change association twice and so
$\nu_1 >0$
We just focused on who $A_1$ matches to.
The remaining tracks are matched to each other so $\nu_2=\nu_3 = 0$.
In short $\mathcal{D}_{MOTA}(A,B) = \mu > 0$.

\end{proof}

\end{document}